\documentclass{article}

\usepackage{microtype}
\usepackage{graphicx}
\usepackage{subfigure}
\usepackage{booktabs} %

\usepackage{dblfloatfix} %

\usepackage{hyperref}

\usepackage[accepted]{icml2024}

\usepackage{amsmath}
\usepackage{amssymb}
\usepackage{mathtools}
\usepackage{amsthm}

\usepackage[capitalize,noabbrev]{cleveref}

\usepackage[utf8]{inputenc} %
\usepackage[T1]{fontenc}    %
\usepackage{hyperref}       %
\usepackage{url}            %
\usepackage{booktabs}       %
\usepackage{amsfonts}       %
\usepackage{nicefrac}       %
\usepackage{microtype}      %
\usepackage{xcolor}         %
\usepackage{graphicx}
\usepackage{comment}
\usepackage{amsmath}
\usepackage{wrapfig}
\usepackage{svg}
\usepackage[algoruled,noend]{algorithm2e} %
\usepackage{adjustbox}
\usepackage{placeins}

\usepackage{float}
\usepackage{lipsum}

\usepackage{rotating}
\usepackage{graphicx}

\usepackage{tabularx}
\usepackage{makecell} %

\theoremstyle{plain}

\usepackage{xcite}

\usepackage{xr}
\makeatletter
\newcommand*{\addFileDependency}[1]{%
  \typeout{(#1)}%
  \@addtofilelist{#1}%
  \IfFileExists{#1}{}{\typeout{No file #1.}}%
}
\makeatother

\DeclareMathOperator*{\argmax}{argmax}
\DeclareMathOperator*{\E}{\mathbb{E}}
\DeclareMathOperator*{\Cov}{\text{Cov}}
\DeclareMathOperator*{\MoGr}{\text{MoGr}}

\theoremstyle{plain}
\newtheorem{theorem}{Theorem}[section]
\newtheorem{proposition}[theorem]{Proposition}

\theoremstyle{definition}

\theoremstyle{remark}

\usepackage[textsize=tiny]{todonotes}

\icmltitlerunning{SBMI: Simulation-Based Model Inference}

\begin{document}

\twocolumn[
\icmltitle{Simultaneous Identification of Models and Parameters of Scientific Simulators}

\icmlsetsymbol{equal}{*}

\begin{icmlauthorlist}
\icmlauthor{Cornelius Schr\"oder}{AIC}
\icmlauthor{Jakob H. Macke}{AIC,MP}

\end{icmlauthorlist}

\icmlaffiliation{AIC}{Machine Learning in Science, University of Tübingen and Tübingen AI Center, Germany}
\icmlaffiliation{MP}{Max Planck Institute for Intelligent Systems, Department Empirical Inference, Tübingen, Germany}
\icmlcorrespondingauthor{Cornelius Schr\"oder, Jakob Macke}{firstname.lastname@uni-tuebingen.de}

\icmlkeywords{Machine Learning, ICML, Simulation-based Inference, Model Learning}

\vskip 0.2in
]

\printAffiliationsAndNotice{}

\begin{abstract}
Many scientific models are composed of multiple discrete components, and scientists often make heuristic decisions about which components to include.
Bayesian inference provides a mathematical framework for systematically selecting model components, but defining prior distributions over model components and developing associated inference schemes has been challenging.
We approach this problem in a simulation-based inference framework: We define model priors over candidate components and, from model simulations, train neural networks to infer joint probability distributions over both model components and associated parameters. Our method, simulation-based model inference (SBMI), represents distributions over model components as a conditional mixture of multivariate binary distributions in the Grassmann formalism. SBMI can be applied to any compositional stochastic simulator without requiring likelihood evaluations. We evaluate SBMI on a simple time series model and on two scientific models from neuroscience, and show that it can discover multiple data-consistent model configurations, and that it reveals non-identifiable model components and parameters. 
SBMI provides a powerful tool for data-driven scientific inquiry which will allow scientists to identify essential model components and make uncertainty-informed modelling decisions.
\end{abstract}

\section{Introduction}

Computational models are a powerful tool to condense scientific knowledge into mathematical equations. These models can be used for interpreting and explaining empirically observed phenomena and predicting future observations. 
Scientific progress has always been driven by competing models, dating back to disputes about the heliocentric system \citep{copernicus1543revolutionibus}. However, newly developed models are rarely that disruptive; instead, they are often created by combining existing components into larger models. %
For example, the original SIR model \citep{kermack1927SIR} describes the dynamics of infectious diseases by three population classes (\textbf{s}usceptible, \textbf{i}nfective, \textbf{r}ecovered), but was later expanded to  include further epidemiological classes \citep[e.g., temporary immune groups, ][]{hethcote2000SIR}. 
Similar modularity can be found, for example,  in computational neuroscience models: The original Hodgkin-Huxley model \citep{hodgkin1952quantitative} for the dynamics of action potentials 
consisted of only two voltage-gated ion channels ($K^+$, $Na^+$), but more recent models \citep{mccormick1992model,pospischil2008minimal} are based on compositions of a myriad of different channels \citep{podlaski2017mapping}.
Similarly, there exist many variants of drift-diffusion models (DDM) \citep{ratcliff1978theory} in cognitive neuroscience:  All of them follow the basic concept of modeling the decision process by a particle following a stochastic differential equation and eventually hitting a decision-boundary. There are many possible choices of noise models, drift dependencies, and boundary conditions. This rich model class and many of the different components have been extensively studied on a wide range of experimental measurements \citep{,ratcliff2008diffusion_review,
latimer2015single, turner2015informing}.

\begin{figure*}[t]
  \centering
  \includegraphics[width=0.9\textwidth]{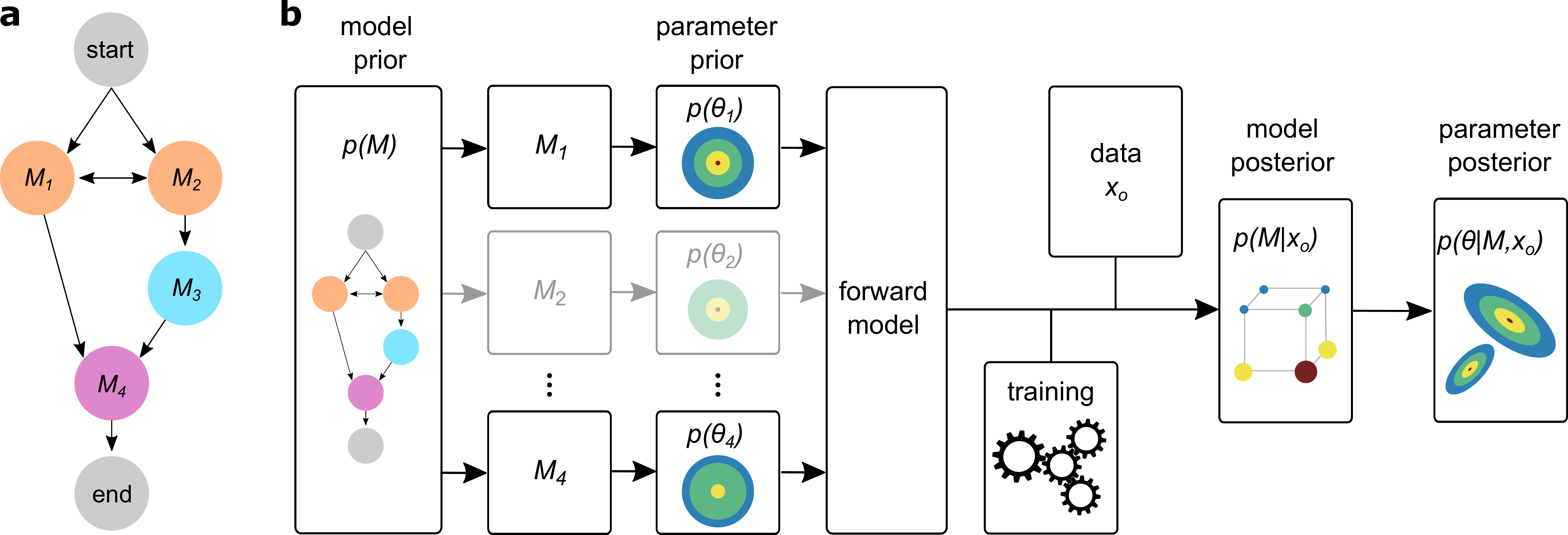}
  \caption{
    \textbf{Simulation-based model inference (SBMI) scheme. (a)} The model prior  $p(M)$ is given implicitly by a graph. A random walk from the \emph{start} to the \emph{end} node corresponds to a draw from this prior. 
    \textbf{(b)} We first sample from the model prior and the corresponding parameter priors $p(\theta_i)$ to compile a forward model. Following this sampling procedure, we generate training data with which we can learn a approximation of the joint posterior $p(M,\theta|x)$ by factorizing the posterior into $p(M|x) p(\theta|M,x)$. Finally we can evaluate this posterior for some observed data $x_o$. 
    \vspace{-12pt}
    }
    \label{fig:inference_scheme}
\end{figure*}

How can one automatically infer such models from data, including \emph{both} the compositions of components and the associated parameters? 
One challenge is posed by the fact that, for many such models, evaluating the likelihood function is not tractable, rendering standard likelihood-based approaches  inapplicable.  
Approximate Bayesian computation (ABC) \citep{sisson2018handbook}, offers a framework to deal with this challenge in a systematic way, and in the last years, the development of new methods has been fueled by advances in neural network-based density estimation \citep{papamakarios2021normalizing} leading to new simulation-based inference (SBI) methods \citep{papamakarios2016fast,lueckmann2017flexible,  cranmer2020frontier}. 
SBI has been successfully applied to various fields like astronomy \citep{dax2022group}, robotics \citep{marlier2021simulationrobotics}, neuroscience \citep{gonccalves2020training,deistler2022energy,groschner2022biophysical} and cognitive science  \cite{radev2020bayesflow,boelts2022flexible}.

However, in addition to inferring parameters, we also need to be able to compare and select models comprised of different components to select between competing theories.  Standard methods for Bayesian model comparison (or selection) rely on the \emph{Bayes factor} \citep{kass1995bayes}, i.e., the ratio of model evidence for two different models $M_1$ and $M_2$: $B_{12}:={p(x_o|M_1)}/{p(x_o|M_2)}.$ %
Multiple approaches have been developed for estimating Bayes factors, most of which are  based on (rejection) sampling \citep{trotta2008bayessky} and are computationally expensive. 
Alternative approaches include approximating the model evidence by applying harmonic mean estimators to likelihood emulators  \citep{mancini2022bayesian}, or by directly targeting the model posteriors in an amortized manner \cite{boelts2019comparing,radev2021amortizedmodelcomparison}. 
While these methods infer the model evidence separately for each model or assume a fixed set of models to compare, our approach allows for a comparison of flexible combinations of model components in a fully amortized manner.

\emph{Symbolic regression} approaches  aim to learn interpretable mathematical equations from observations--- while this might seem like a conceptually very different problem, it is methodologically related, as one can also interpret mathematical equations as being composed of different model components.  Inferring symbolic equations from data can be tackled by genetic programming \citep{schmidt2009distilling_eureqa, dubvcakova2011eureqa}, by performing sparse regression over a large set of base expressions \cite{brunton2016discovering, bakarji2022discovering} or by using graph neural networks \citep{cranmer2020discovering}. Alternatively, symbolic regression has been approached by designing neural networks with specific activation functions \citep{martius2016extrapolation,sahoo18a}, optimizing these networks with sparsity priors \citep{werner2021informed} and using Laplace approximations to infer uncertainties over their weights \cite{werner2022uncertainty}.
Building on the success of transformers, \citet{biggio2021neural} introduced a transformer-based 
approach for symbolic regression, which was recently extended to capture differential equations \citep{becker2022discovering}.

Our work builds on these advances in both SBI and symbolic regression. However, our goal is to infer \emph{joint} posterior distributions over a set of different model components, \emph{as well as} over their associated parameters. One can interpret our approach as performing fully probabilistic symbolic regression not on `atomic' symbols, but rather on expression `molecules' which are provided by domain experts and represent different mechanisms that might explain the observed data.
As we will show, accurate inference of joint posteriors is crucial for obtaining interpretable results in the presence of redundant model components: A common situation in scientific applications is that different components are functionally similar \citep[e.g., ion channels with similar dynamics, ][]{podlaski2017mapping}, resulting in explaining-away effects and strongly correlated posterior distributions. Hence, inference methods need to be able to accurately handle such settings to obtain scientifically interpretable results.

We address this challenge by providing a network architecture for joint inference, which includes a flexible representation over model components using mixtures of multivariate binary distributions in the Grassmann formalism \citep{arai2021grassmann}. Second, for such a procedure to be able to provide parsimonious results, the ability to flexibly specify priors over models is crucially important. Our procedure only requires the ability to generate samples from the prior \citep[like ][]{biggio2021neural}, without requiring access to evaluations of prior probabilities. Third, our approach is fully \emph{amortized}: Once the inference network has been trained, approximate posteriors over both model components and parameters can be inferred almost instantly, without any computationally expensive MCMC sampling or post-hoc optimizations at inference-time.

In the following, we first describe our inference method (Sec. \ref{sec:method}) and showcase it on an additive model related to symbolic regression (Sec. \ref{subsec:additive_model}). We then apply it to DDMs and experimental decision-making data (Sec. \ref{subsec:ddm}) as well as to Hodgkin-Huxley models and voltage recordings from the Allen Cell Types database \cite{Allen2016}  (Sec. \ref{subsec:HH}) and show that our method can successfully retrieve interpretable posteriors over models.

\section{Method} \label{sec:method}

\begin{figure}
\centering
  \includegraphics[width=0.4\textwidth]{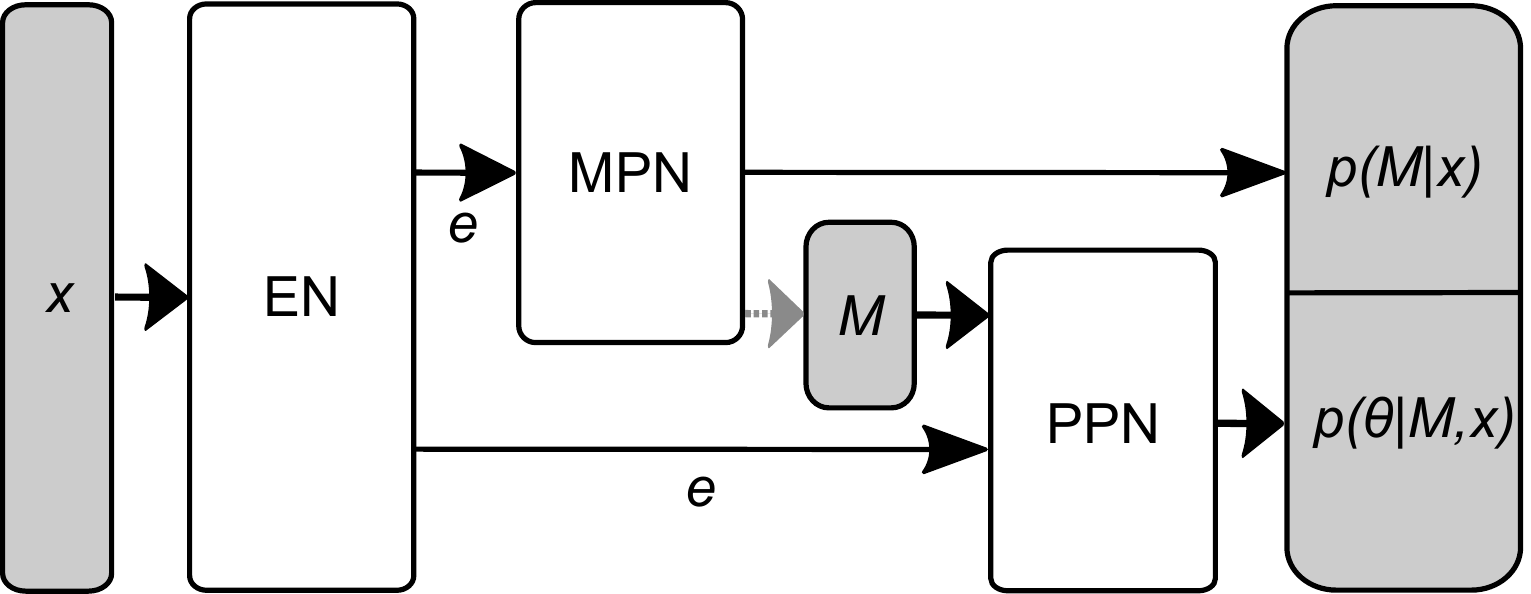}
  \caption{
    \textbf{SBMI network architecture.} Data $x$ is passed through an embedding net (EN). The embedded data $e$ is forwarded to the model posterior network (MPN), which learns posteriors over different model components, and the parameter posterior network (PPN) which learns the posterior distributions over parameters given specific models $M$. Gray boxes correspond to network inputs / outputs. %
    \vspace{-14pt} %
  }
  \label{fig:network_architecture}
\end{figure}

Our proposed method, \emph{simulation-based model-inference} (SBMI), performs inference over a model $M$  consisting of different model components $M_i$ and their associated parameters $\theta_i$. More specifically, we use a neural posterior estimation (NPE) method to approximate a joint posterior distribution $p(M,\theta|x_o) = p(M|x_o) p(\theta|M,x_o)$ given some observed data $x_o$ end-to-end (Fig.~\ref{fig:inference_scheme}). While we take a related approach to previous NPE methods \cite{papamakarios2016fast,lueckmann2017flexible,greenberg2019automatic}, we lift the assumption of a fixed simulator and include the inference over model components into our framework. We therefore assume that we have a `black-box' model from which we can draw samples $x_j \sim p(x | M,\theta)$, but don't necessarily have access to the likelihood, any other internal states, or gradients of the model.  Approximate Bayesian inference is performed by first generating simulations which are then used to learn posterior distributions.  These can be evaluated in an amortized manner for new observations $x_o$ to get the full joint posterior $p(M,\theta|x_o)$.

\subsection{Priors and Data Generation}
To allow maximal flexibility in designing appropriate priors, SBMI only requires access to an implicit prior distribution from which we can sample models. 
We here define the model prior by a directed graph with dynamically changing weights, defined as a triplet $(\mathcal{M},\mathcal{E}, \mathcal{R})$. 
The set of vertices $\mathcal{M}=\{M_i\}_{i \in 0,...,N+1}$ corresponds to the model components and additional start/end node $M_0$ and $M_{N+1}$ (Fig.~\ref{fig:inference_scheme}a). 
The set of edges $\mathcal{E} = \{e_{ij}=(M_i, M_j,w_{ij}) | M_i, M_j \in \mathcal{M}, w_{ij} \in \mathbb{R}_{\geq 0} \}$ are directed, with weights $w_{ij}$. 
To sample from the prior, we perform a random walk on the graph with $p(M_i|M_j) = \frac{w_{ij}}{\sum_l w_{il}}$ and represent each sampled model $M=(M_1,..., M_N)$ as an ordered binary vector of length $N$.
The set of rules $\mathcal{R}=\{R_k|R_k:(\mathcal{E},S) \to \mathcal{E}, ~S \subset \mathcal{M}\}_{k \in K}$, with index set  $K$, defines how the weights $w_{ij}$ are updated during a random walk. We assume the graph to be \emph{conditionally acyclic}: While the initial edges $\mathcal{E}$ can include cycles (e.g. Fig. \ref{fig:additive_model}a), the updating rules $\mathcal{R}$ ensure that no cycle occurs in prior samples. 
By changing the edge weights we can encode prior knowledge, for example, to favour simple models over complex models, or to encourage (or discourage) the co-occurrence of specific model components. Pseudocode for the sampling procedure and example updating rules can be found in Appendix \ref{app:sec:model_prior}.

This graph representation gives us the possibility to flexibly encode prior knowledge of the model by carefully defining its structure and weights with the help of domain expertise.
Once we have sampled a model $M$, we define the prior of the corresponding model parameters as the product of the component-specific priors: $p(\theta|M) = \prod_{i | M_i = 1} p(\theta_i)$, i.e. the parameter vector $\theta$ is of variable size and matches a specific model $M$. The component-specific priors $p(\theta_i)$ can correspond to any continuous, potentially multivariate, distribution. 

To generate training data for learning an approximation of $p(M,\theta|x_o)$ we need a `compiler' that turns the model representation $(M,\theta)$ into a simulator which then generates synthetic data $x$. These compilers and simulators will generally be specific to the model type and based on domain-specific toolboxes. In our numerical experiments, we built a flexible interface to symbolic calculations based on \emph{SymPy} \citep{meurer2017sympy} for the additive model (Sec. \ref{subsec:additive_model}), the \emph{PyDDM} toolbox \citep{shinn2020pyddm} for the drift-diffusion model (Sec.~\ref{subsec:ddm}) and implemented a modularised version of the Hodgkin-Huxley model in \emph{JAX} \citep{jax2018github} (Sec. \ref{subsec:HH}).

\subsection{Inference} \label{subsec:inference}

We want to perform inference over the joint posterior distribution $p(M,\theta|x)$ of the model and its parameters, given some data $x$. We can factorize this distribution $p(M,\theta|x) = p(M|x)p(\theta|M,x)$ and approximate it by jointly learning two coupled network modules (Fig.~\ref{fig:network_architecture}): The first module learns an approximation to the model posterior $q_\psi(M|x) \approx p(M|x)$ and the second one an approximation to the parameter posterior  $q_\phi(\theta|M,x) \approx p(\theta|M,x)$ conditioned on the data and the model.  As the data $x$ might be high-dimensional (or, in principle, of variable length)  we use an additional embedding net to project it to a low-dimensional representation before passing it to the posterior networks. This network can be replaced by `summary statistics' which capture the main features of the data (see Sec. \ref{subsec:HH}).

\paragraph{Model Posterior Network}
To approximate the multivariate binary model posterior $p(M|x)$ we introduce a new family of mixture distributions: mixture of multivariate binary Grassmann distributions (MoGr). Multivariate binary Grassmann distributions were recently defined by \citet{arai2021grassmann}, and allow for analytical probability evaluations. A Grassmann distribution is defined by its probability mass function on a $n$-dimensional binary space, for which closed-form expressions for marginal and conditional distributions are available. This in turn can be directly used for efficient sampling. An $n$-dimensional binary Grassmann distribution $\mathcal{G}$ on $Y = (Y_1,...,Y_n)$ is parameterized by a $n\times n $ matrix $\Sigma$ which is analogous to the covariance of a normal distribution, but not necessarily symmetric. The mean of the marginal distribution is represented on the diagonal and the covariance is the product of the off-diagonal elements: %
\begin{equation*}
    \mathbb{E}[Y_i] = \Sigma_{ii}, \qquad \text{Cov}[Y_i,Y_j] = - \Sigma_{ij} \Sigma_{ji}.
\end{equation*}
For a valid distribution $(\Sigma^{-1}-{I})$ must be a $P_0$-matrix, but has otherwise no further constraints \cite{arai2021grassmann}.
We thus define a mixture of Grassmann distribution as ${\MoGr (Y) =\sum_i \alpha_i \mathcal{G}_i(Y)}$ for a finite partition ${\sum_i \alpha_i = 1}$ and Grassmann distributions $\mathcal{G}_i$. We denote the corresponding conditional distribution by $\MoGr (Y|e) =\sum_i \alpha_i \mathcal{G}_i(Y|e)$, for some real-valued context vector $e$ (which will be the embedded data in our case). Further details (including some key properties, and implementation details) in Appendix \ref{app:grassmann}. 
We trained the model posterior $p(M|x)$ represented as conditional MoGr distribution $\MoGr(M|x) \approx q_\psi(M|x)$ by minimizing the negative log-likelihood. The model loss  $\mathcal{L}_M$ is therefore defined by $\mathcal{L}_M(\psi)  =  - \log q_\psi(M|x)$.

\paragraph{Parameter Posterior Network}
The parameter posterior network  $q_\phi(\theta|M,x)$ needs the flexibility to deal with different dimensionalities, as $\theta$ is only defined when the respective model component ($M_i=1$) that uses $\theta_i$ is included. While recent SBI approaches typically used normalizing flows \citep{papamakarios2021normalizing} for parameter inference, we use a mixture density network (MDN) of Gaussian distributions on the full-dimensional parameter space (with dimension $n=\sum_{i}\text{dim}(\theta_i)$) and marginalize out the non-enclosed model components. This allows the network to learn dependencies across  model components (which is critical, e.g., to account for compensation effects between redundant components).  
We construct this flexible MDN by defining for every $\theta$ its complement $\theta^C$ as the parameter dimensions not present in $\theta$ and $\bar{\theta}=(\theta,\theta^C)$. We further define $\bar{p}$ as the $n$-dimensional distribution acting on $\bar{\theta}$.  We can now define the parameter posterior network $q_\phi(\theta|M,x)$ by  marginalizing out $\theta^C$,
\begin{equation*}
    q_\phi(\theta|M,x) 
    = \int \bar{p}(\bar{\theta}|M,x) d\theta^C .
\end{equation*}
We use the standard NPE loss \citep{papamakarios2016fast} for the parameter posterior network $\mathcal{L}_{\theta}$ : $\mathcal{L}_{\theta}(\phi) = - \log ~q_\phi(\theta|M,x)$. 
The final loss function for the training of the three different network modules (embedding net, model, and parameter posterior network) is then defined as the expected sum of the two posterior losses:
$\mathcal{L}(\psi, \phi) = \frac{1}{L} \sum_{l} \mathcal{L}_{M_l}(\psi) + \mathcal{L}_{\theta_l}(\phi)$, for a batch of training samples $\{(\theta_l,M_l,x_l)\}_l$ of size $L$. 
In Proposition \ref{proof:KL}
we show that optimizing this loss functions minimizes the expected Kullback-Leibler divergence between true joint posterior $p(M,\theta|x)$ and the approximated posterior 
\begin{equation*}
    \mathbb{E}_{p(x)} \big [D_{KL}( p(M,\theta|x)|| q_\phi(M|x)q_\psi(\theta|M,x)\big ],
\end{equation*}
and is therefore retrieving the object of interest.
See Algorithm \ref{alg:SBMI} for pseudocode.
Our implementation is based on the \emph{sbi} toolbox \citep{tejerocantero2020sbi} (see Appendix \ref{app:inference} for details).    

\begin{figure*}[t]
  \centering
  \includegraphics[scale=0.95]{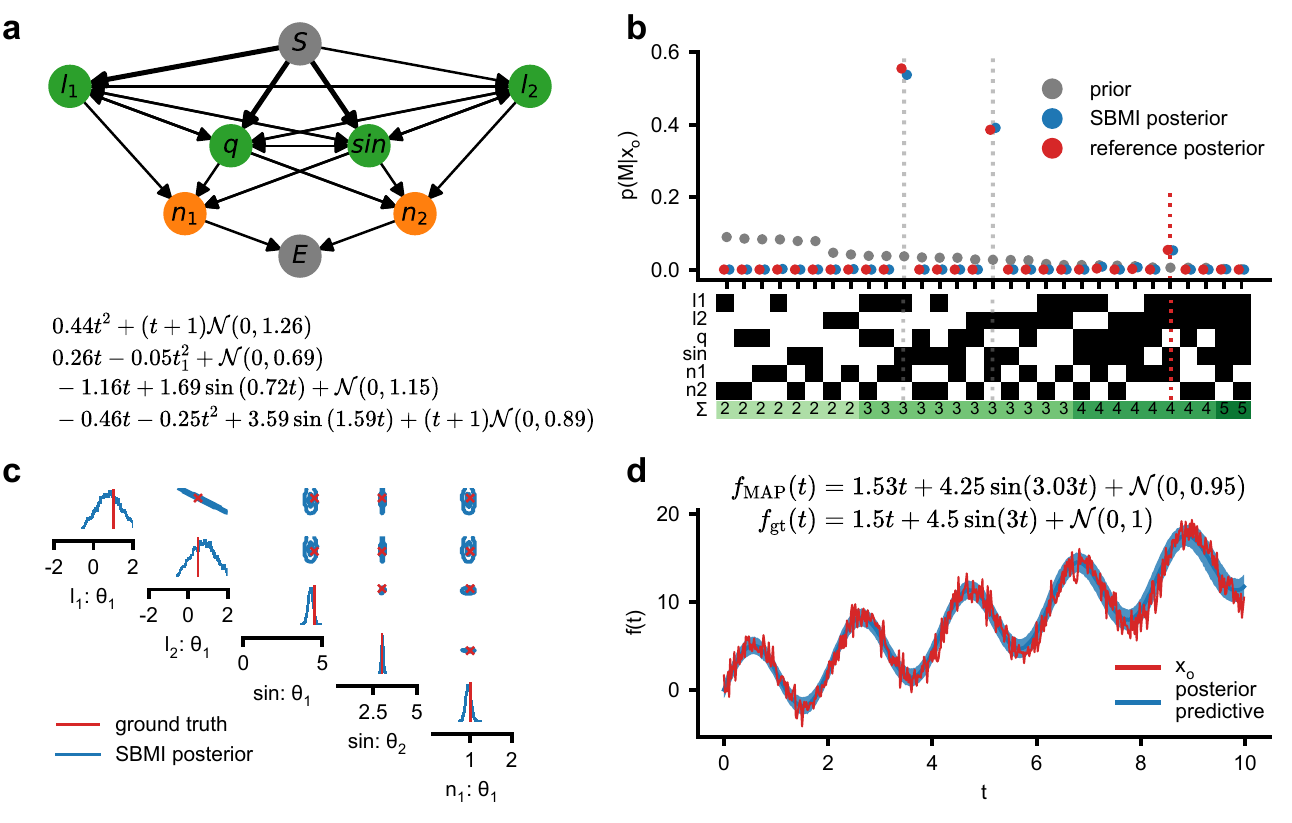}
  \vspace{-6pt}
  \caption{
    \textbf{Additive model. (a)} Model prior represented as a graph, the width of the edges corresponds to their initial weights, which change dynamically. A random walk from \emph{start}  (S) to \emph{end} node (E) corresponds to one draw from the prior. Four prior samples are shown.  
    \textbf{(b)} Empirical prior distribution, reference and SBMI posterior distribution for one example observation, generated by the model highlighted by the red dashed line. The  model vectors are shown as binary image, black indicating the presence of the specific model component. SBMI accurately recovers the posterior over model components.
    Marginal distributions in Fig. \ref{fig:app:additive}.
    \textbf{(c)} One- and two-dimensional marginals of the parameter posterior inferred with SBMI, conditioned on the `true model' (red dotted line in (b)). Note the strongly negatively correlated (degenerate) posterior between the redundant model components $l_1$ and $l_2$. Parameter posteriors for additional models in  Fig. \ref{fig:app:additive}. 
    \textbf{(d)} Predictive samples for an observation $x_o$ from $f_\text{gt}$.  Blue: Mean $\pm$ std. as local uncertainties of the posterior predictives $x \sim p(x|\theta, M)$ with $\theta \sim p(\theta|M,x_o)$.
    \vspace{-6pt}
    }
    \label{fig:additive_model}
\end{figure*}

\paragraph{Local and Global Uncertainties:}
SBMI allows us to calculate two different uncertainties for the posterior predictives, depending on whether uncertainty about model-choice is taken into account or not: \emph{Local} uncertainties \citep{werner2022uncertainty} are defined as the uncertainty of parameter posteriors conditioned on a specific model $M_i$: 
$ x \sim p(x|M_i, \theta) \text{~ with~ } \theta \sim p(\theta|M_i, x_o)$.
In contrast, for the \emph{global} uncertainty, the joint posterior is taken into account:
$ x \sim p(x|M, \theta) \text{~with~} M,\theta \sim p(M,\theta|x_o)$.

\paragraph{Simulation-based Calibration}
To validate the inferred parameter posteriors we perform simulation-based calibration (SBC) on the marginal statistics \cite{talts2018validating}. In SBC each marginal of the true parameter $\theta_o$ is ranked according to the marginals of samples from the parameter posterior $p(\theta|x_o)$ for a simulated data sample $x_o \sim p(x|\theta_o)$. For a well calibrated posterior, the ranks follow a uniform distribution, which we evaluated by a classifier-two-sample-test \citep[c2st, ][]{friedman2003multivariate}.

\section{Experiments}

We demonstrate SBMI on three model classes: An illustration on an additive model of a one-dimensional function $f(t)$, variants of the \emph{drift diffusion model} (DDM) from cognitive science and variants of the \emph{Hodgkin-Huxley model}. %

\subsection{Additive Model} \label{subsec:additive_model}

For the additive model, we used  two linear, a quadratic, a sinusoidal, and two different noise terms (details in Table \ref{tab:additive_model}), all evaluated on an equidistant grid on the interval $[0,10]$. These could be seen as the `base functions' in a symbolic regression task. 
To investigate how SBMI fares in the presence of non-identifiability, we included two identical linear components which only differ in their prior probability. %
We defined the model prior as a dynamically changing graph (Fig. \ref{fig:additive_model}a) which favors simpler models (Fig. \ref{fig:additive_model}b, Appendix \ref{app:model_details_additive}).
We used a CNN followed by fully connected layers for the embedding net (Appendix \ref{app:model_details}). We generated a dataset of 500k prior samples, of which 10\% were used as validation data. 

In the presented model, we have access to the likelihood function $p(x_o|M,\theta)$, and can approximate the model evidence via (importance) MC sampling. We approximate the model evidence $p(M|x_o) \approx \hat p_{\text{reference}}(M|x_o) \sim p(x_o|M)p(M)$ by sampling for each model $M$ corresponding parameters $\theta_i^j$ and evaluating the likelihood $p(x_o|M)$ (Appendix \ref{app:performance measures}).  We call the resulting approximation \emph{reference posterior}, and will use it to evaluate the accuracy of the posterior inferred by SBMI. As the parameter space for $\theta$ can be high-dimensional, and the corresponding posterior distribution $p(\theta|M,x_o)$ can be narrow, a reliable numerical approximation needs an extensive amount of samples and model evaluations for each of the model combination $M$. It is therefore not feasible for larger model spaces. 

Across 100 observations $x_o$ for which we computed  reference posteriors, the Kullback–Leibler divergence (KL) between the reference posterior and the SBMI posterior $\text{KL}(\hat p_{\text{reference}}(M|x_o)||q_\psi(M|x_o))$ decreased substantially with the number of training samples. %
When we replaced the MoGr distribution by a `flat' categorical distribution, and left all other components unchanged, the inference is much less data efficient (Fig. \ref{fig:additive_comparison_cat}a).
Additionally, we compared the marginal distribution of the model posterior to the ground truth model. The performance of the marginal model posterior distributions inferred by SBMI is very similar to that of the reference posteriors
(Table \ref{tab:additive_performance}) with a correlation of 0.85 (Fig. \ref{fig:app:additive_correlation_mmp}). 
We note that, in initial  experiments in which we used masked autoregressive density estimators (MADE) \citep{germain2015made} (instead of the Grasmann mixtures) exhibited worse performances in comparison (Fig. \ref{app:fig:made_results}), indicating the power and flexibility of MoGr distributions.

For the evaluation of the joint posterior $p(M,\theta|x_o)$, we focused on the evaluation in the predictive (data) space. 
To this end, we sampled models $M_l \sim q_\psi(M|x_o)$ and associated parameters $\theta_l \sim q_\phi(\theta|M_l,x_o)$ from the inferred posteriors for 1k test observations $x_o$ and ran the forward model $x_l \sim p(x|M_l,\theta_l)$. Based on these simulations, we calculated the root-mean-squared-error (RMSE) of the simulations $x_l$ to the observed data $x_o$.
The average RMSE between posterior predictive samples follow the same trend as the KL divergence and approach the RMSE between the observations $x_o$ and samples with the same ground truth parameters at around 200k training samples (Fig. \ref{fig:additive_comparison_cat}b). 

Next, we showcase SBMI for a specific example observation $x_o$ in which the ground truth model has two linear, a sinusoidal, and a stationary noise component (Fig. \ref{fig:additive_model}b-d): The SBMI model posterior matches nearly perfectly the reference posterior and predicts the linear components as expected, ordered by the prior probabilities (Fig. \ref{fig:additive_model}b).
The parameter posterior obtained with SBMI and conditioned on the ground truth model accurately recovers the ground truth parameters (Fig.~\ref{fig:additive_model}c). Accessing the joint posterior distribution enables us to first see the perfectly correlated parameter distribution for the slope parameter of the linear components. Second, we detect compensations mechanisms for a model which contains only one linear component: In this case, the predicted parameter for $l_1$  is the sum of the ground truth parameters of $l_1$ and $l_2$, resulting in the same functional expression (Fig.~\ref{fig:additive_model}d). 
For the posterior predictives (Fig.~\ref{fig:additive_model}d) we see that most of the observed data $x_o$ lies within an uncertainty bound of one standard deviation around the mean prediction. 
The local uncertainties overlap almost perfectly in this case, as all models with non-negligible posterior mass have the same expressional form. 
With the inferred model posterior we can easily compute the Bayes factors via $\frac{p(M_1|x_o)}{p(M_2|x_o)} \frac{p(M_2)}{p(M_1)}$ to compare different models on an observation $x_o$, and an example for Fig. \ref{fig:additive_model} is shown in Appendix \ref{app:Bayes_factor_additive}.

\begin{figure}[t]
  \centering
  \includegraphics[width=0.5\textwidth]{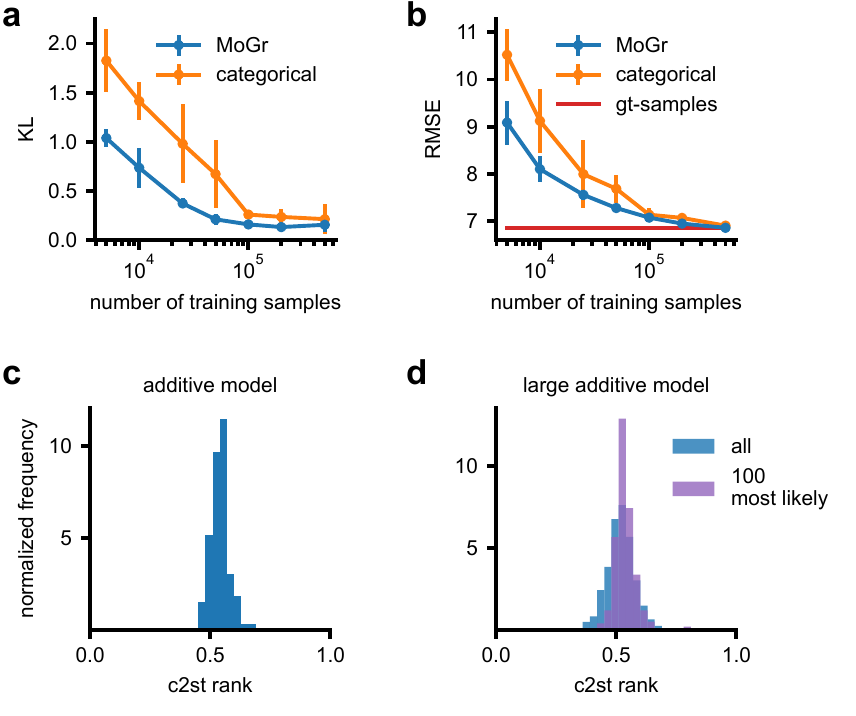}
  \vspace{-18pt}
  \caption{
  \textbf{SBMI performance for the additive model.}  %
    \textbf{(a)} KL divergence of the SBMI model posterior $q_\psi(M|x_o)$ to the reference posterior $\hat{p}_\text{ref}(M|x_o) $ for 100 observations $x_o$. 
    \textbf{(b)} Posterior predictive performance in terms of RMSE between 1k observations $x_o$ and posterior samples $x_l$. Red line indicates RMSE between $x_o$ and samples from the ground truth (gt) model as lower bound.
    (a) and (b) show mean and std. for 5 training runs and different numbers of training samples (from 5k to 500k).
    \textbf{(c)} Histogram of the c2st ranks for the additive models with six components with a c2st mode of 0.54 (0.48/0.61 as .05/.95 percentiles). A value of 0.5 indicates a well calibrated posterior for which the rank statistics are indistinguishable from a uniform distribution.
    \textbf{(d)} Same as (c)  for the additive model with eleven components with a c2st mode of 0.52 (0.43/0.62, `all') and 0.53 (0.48/0.61, `100 most likely').
    See also Fig. \ref{fig:app:SBC_additive} for individual SBC plots. 
    \vspace{-6pt}
    }
    \label{fig:additive_comparison_cat}
\end{figure}

In two additional sets of experiments we increased the space of model components to eleven and twenty  (see Tab. \ref{tab:additive_modelBIG} and Tab. \ref{tab:bigger_model} for details). As calculating reference posteriors is not computationally feasible any more,  we focused on the evaluation in the predective space and calculated in the first set the RMSE for different numbers of training samples for MoGr as well as for categorical model posterior distributions. While the MoGr did almost reach the lower bound with 1M training samples, the categorical distribution had a much worse performance and did not reach this lower bound. When we increased the number of mixture components for the MoGr as well as for the Gaussian MDN from three up to twelve, the RMSE did not change substantially, except that less mixture components were preferable in a low data regime (Fig. \ref{fig:app:additive_BIG_comparison_cat}a). 
In the second set of experiments with twenty model components, the categorical distribution would need to learn all  $2^{20} \approx 1M$ possible combinations, which is not feasible any more. However, the MoGr was still able to learn the posterior distributions and the predictive performance reached almost the lower bound with 1M training sample (Fig. \ref{fig:app:additive_BIG_comparison_cat}b).

When we applied SBC to the parameter posteriors of the additive model, which were conditioned on the “ground truth model”, we found well calibrated posterior distributions for all 30 possible models and parameters (Fig. \ref{fig:additive_comparison_cat}c). This still holds true for the large additive model, for which we included all models with at least 50 samples in the test dataset of 100k samples (Fig. \ref{fig:additive_comparison_cat}d). When we only looked at the 100 most likely models the c2st values were even more tightly centered around one half. We found no systematic bias for individual parameters in our inference method and most posterior ranks fall into the 99\% confidence interval of a uniform distribution (Fig. \ref{fig:app:SBC_additive}). These results show that the parameter posteriors are well calibrated for the additive model, even for models which are less likely under the prior distribution and the posteriors reflect the underlying uncertainty  well. 

\begin{figure*}[th]
  \centering
  \includegraphics[scale=0.95]{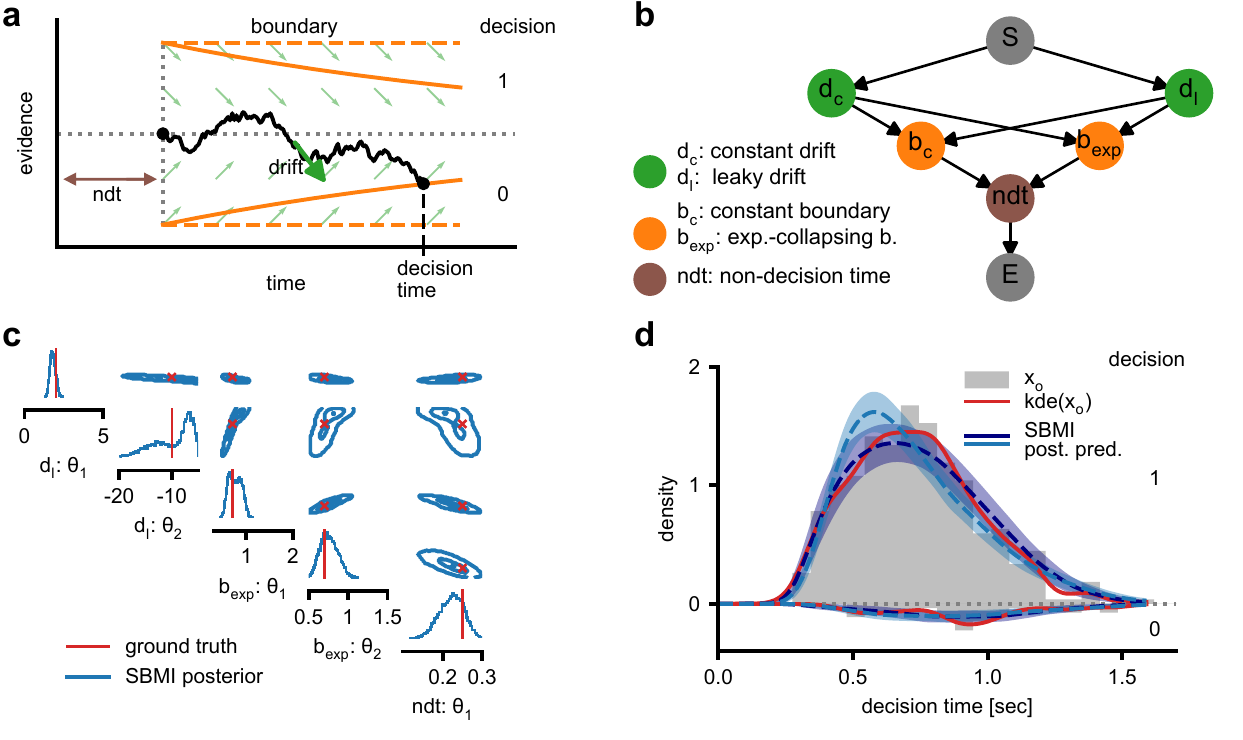}
  \vspace{-6pt}
  \caption{
    \textbf{SBMI on Drift-Diffusion Models. (a)} A decision process is modelled by a one-dimensional stochastic process. A binary decision is taken once the process hits the upper or lower boundary, resulting in a two-dimensional output (a continuous decision time and a binary decision). 
    \textbf{(b)} The model prior is a graph consisting of two drift ($d_c$, $d_l$) and two boundary ($b_c$, $b_{exp}$) components, as well as a non-decision time ($ndt$).
    \textbf{(c)} Example parameter posterior inferred with SBMI for which both, the ground truth model and the predicted model, have leaky drift and exponentially collapsing boundary conditions.
    \textbf{(d)} Posterior predictives with local uncertainties as mean $\pm$ std. for the two most likely models (dark blue with $q_{\psi}(M|x_o)=0.75$ and light blue with $q_{\psi}(M|x_o)=0.25$).
    \vspace{-6pt}}
    \label{fig:ddm}
\end{figure*}

\subsection{Drift-Diffusion Model} \label{subsec:ddm}

After this illustrative example, 
we turn to  DMMs, a scientific model class that we will apply to experimental data.
DDMs can be described by a stochastic differential equation for a decision variable $z$:
$\text{d}z = d(z,t)\text{d}t + \text{d}W$, with initial condition $z_0$, drift term $d$, and a Wiener noise process $W$. A decision is taken when the decision variable hits a boundary $|d(z,t)|\geq b(t)$ (Fig.~\ref{fig:ddm}a). An additional parameter delays the starting time of the process (`non-decision time').  We included two different drift terms (constant and leaky), two boundary conditions (constant and exponentially collapsing), and the non-decision time to our prior  (Fig.~\ref{fig:ddm}b, Appendix \ref{app:model_details_ddm}), resulting in a  highly flexible model class. Similar models have previously been applied to experimental data \citep{shinn2020pyddm}. 

Training data was generated with the \emph{pyDDM} toolbox \citep{shinn2020pyddm} and for each $\theta$ we sampled 400 identically distributed (iid) trials trial. These were embedded by a permutation invariant embedding network \citep{chan2018likelihood,radev2020bayesflow} (Appendix \ref{app:model_details}).

For the DDM, we don't have efficient access to the likelihood and therefore cannot compute reference posteriors. To still evaluate the performance of SBMI, we focus on the evaluation of model posteriors and predictive performances for a test set of 1k data points. The average marginal performance of the model posterior for the drift and boundary components is 0.87$\pm$0.21 (std.) (see Table \ref{tab:ddm_details} for individual performances). For about 40\% of the test data we get highly certain model posteriors with $p(M|x_o)>0.99$, indicating that model identifiability is dependent on the observed data $x_o$.
To measure the performance of the posterior predictives, we compared the mean decision times, the standard deviation of the decision times, and the number of correct trials to the observed data $x_o$. Additionally, we used the mean-squared error (MSE) on the weighted density functions of the two different decisions, similar to \cite{shinn2020pyddm}. The different measures on the posterior predictives for the test data are close to their lower bounds (Table \ref{tab:ddm_performance}), calculated on trials resampled from a model with the ground truth parameter. This suggests that, even for non-identifiable models, the SBMI inferred posterior predictives are close to data from the ground truth model. 

\begin{table}
\small 
\caption{\textbf{DDM posterior predictive performance.} Comparison of decision times (mean $\mu$ and std. $\sigma$) of the ground truth data~($\hat{\cdot}$) to posterior predictive samples.  The lower bound is based on resampling 400 trials with the ground truth parameters. Mean and standard deviation for 1k test points.}
\centering
\begin{tabular}{lll}
\toprule
                 Measure &  Lower bound &    Posterior \\
\midrule
    decision time: $| \mu - \hat \mu|$ & 0.03 (0.04) & 0.08 (0.21) \\ 
    decision time: $| \sigma - \hat \sigma|$ & 0.21 (0.27) & 0.26 (0.35) \\
    deviation correct trials in \%  & 1.10 (1.37) &  1.56 (2.64) \\
    MSE on densities ($\cdot 10^{-2}$) & 3.14 (5.83) & 3.23 (6.57) \\ 
\bottomrule
\end{tabular}
\label{tab:ddm_performance}
\vspace{-12pt}
\end{table}

For an example observation $x_o$ from a model with a leaky drift component and exponential collapsing boundaries, the `true' model has a posterior probability of 0.75 and a model with a constant drift instead has a posterior probability of 0.25, resulting in a Bayes factor of $B=2.32$, or a `barely worth mentioning' difference \citep{jeffreys1998theory}.  For the `true' model the ground truth parameters lie in regions of high parameter posterior mass, with some uncertainty, especially in the leak parameter $\theta_2$ of the drift component (Fig.~\ref{fig:ddm}c). The posterior predictives match the data well if conditioned on the `true' model. For the model with the constant drift term, we see a slight skew to earlier decision times, compared to the model with leaky drift (Fig. \ref{fig:ddm}d). If we inspect the global uncertainties (Fig. \ref{fig:app:ddm}d) we see a good correspondence for the global uncertainties, also reflected in the MSE losses (scaled by $10^2$): For trials resampled with the ground truth parameters we find an MSE of 0.57$\pm$0.13 which matches the MSE of the first model (0.58$\pm$0.14) and the second model is only slightly worse  (0.61$\pm$0.14).
Further inspecting the posterior distributions shows that the model with the constant drift term exhibits shorter non-decision times, larger initial boundaries and faster collapsing boundaries (Table \ref{tab:ddm_example_obs}). Interpreting the inferred values model-independent as behavioral variables can therefore be difficult, as different models might lead to different inferred values. 

\paragraph{DDM on Experimental Data}

To demonstrate SBMI on empirical data, we used a published dataset of perceptual decision-making data from monkeys \citep{roitman2002data} performing a random dot motion discrimination task.  Moving dots with different coherence rates (0, 3.2, 6.4, and 12.8\%) were visually presented and animals had to identify the direction of movement (Appendix \ref{app:model_details_ddm}). 

\begin{figure}[h]
  \centering
\includegraphics[width=0.4\textwidth]{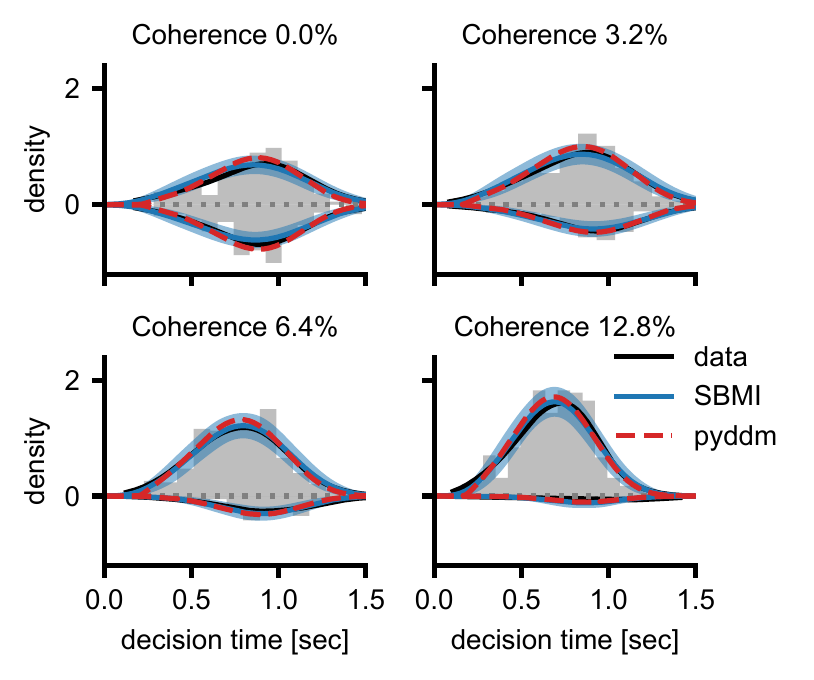}
   \vspace{-6pt}
  \caption{
    \textbf{SBMI on experimental data.} Experimental data with histograms (grey), mean posterior predictives $\pm$ 2std.  and \emph{pyDDM} fits for different coherence rates.
    \vspace{-6pt}
    }
\label{fig:ddm_experimental}
\end{figure}

When we used the trained posterior network to perform amortized inference on the different experimental conditions of the experimental data, the model posterior is certain about the leaky drift and exponentially collapsing boundary component with $p(M|x_o)\approx 1$  for all coherence rates. For all measures on the posterior predictives we found similar mean performances for the SBMI inferred models compared to point estimates (Table \ref{tab:ddm_experimental}). But, as expected, the MSE had higher variances in the different experimental conditions compared to the variance of multiple point estimates (Table \ref{tab:ddm_experimental_detailed}). This can also be seen in the decision time densities of the posterior predictives for which the experimental data lies within the uncertainty bounds (Fig.~\ref{fig:ddm_experimental}), whereas the predictives of the point estimates from \emph{pyDDM} are not distinguishable. 
However, in the parameter space we see that different point estimates  are spread out for some of the parameters, and all lie in regions of high SBMI parameter posterior mass. An example of the two-dimensional marginals for the coherence of 6.4\% is shown in  Figure \ref{fig:app:ddm_2d_pyddm}.

\subsection{Hodgkin-Huxley Model} \label{subsec:HH}

Finally, we apply SBMI to the Hodgkin-Huxley model, a biophysical model for spike generation in neurons. We include a leakage current, four different voltage-gated ion channel types ($Na$, $K$, $K_m$, $Ca_L$) and a noise term (Appendix \ref{app:HH_details}). We encode the domain knowledge that $Na$ and $K$ channels are necessary for  spike generation, resulting in a simple prior (Fig. \ref{fig:app:HHprior}), which could be easily extended to further channel types. 
We replace the embedding net by commonly used summary statistics \cite{gonccalves2020training, scala2021phenotypic}, but leave the inference networks unchanged. 

When we evaluated SBMI on 1k synthetic test traces, we found a mean marginal performance of the model posterior of 0.85$\pm$0.16 (std.) for the two essential model components $K_m$ and $Ca_L$. This performance rises drastically to 0.96$\pm$0.03 if we only include voltage traces with spikes in the test data ($n=439$), indicating that the model components are well identifiable if spikes are present. 
When inspecting the parameter posteriors for the presented examples, the ground truth parameters lie in regions of high posterior mass (Fig. \ref{fig:app:HH_posterior}a).
For the MSE on the posterior predictives' normalized summary statistics we get  0.06$\pm$0.09 for the  example traces shown in Fig. \ref{fig:HH_main} and \ref{fig:app:HH_traces}a, indicating a good performance in the predictive space.

\begin{figure}[th]
  \centering
  \includegraphics[width=0.4\textwidth]{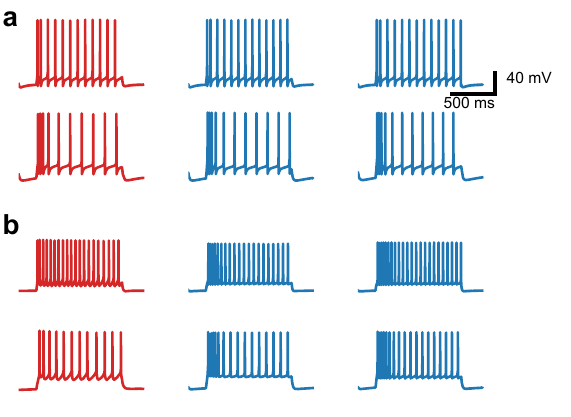}
  \caption{
    \textbf{Posterior predictives of the Hodgkin-Huxley model.} \textbf{(a)} Two synthetic samples (red) with two posterior predictive samples each (blue).  \textbf{(b)} Two voltage recordings from the Allen Cell database with two posterior predictive samples each. See Fig. \ref{fig:app:HH_traces} for more samples, Fig. \ref{fig:app:HH_susmstats}c,d for example summary statistics and Fig. \ref{fig:app:HH_posterior} for SBMI posteriors.
    \vspace{-6pt}
    }
\label{fig:HH_main}
\end{figure} 

To apply SBMI to experimental voltage recordings we took ten voltage traces from the Allen Cell Types database \cite{Allen2016} previously used in \cite{gonccalves2020training}. The model posterior identified $Ca_L$ channel in some (4 out of 10) of the recordings, which were previously not used to fit these traces. The posterior predictives look slightly worse than for the synthetic data, but still capturing the main features (Fig. \ref{fig:HH_main}b), resulting in a MSE of 0.24 $\pm$0.11 (std.) on normalized summary statistics. 

To showcase the influence of the model prior, we run the same SBMI inference scheme with a fully connected model prior, and therefore omitted any domain knowledge for the model structure which we had included before. First of all, for a training dataset of 100k samples this resulted in only 31k spiking models (compared to 41k in the initial experiment). After training, SBMI shows still reasonable performance for most synthetic samples (MSE of $0.10 \pm 0.14$ on the summary statistics). Additionally, for all observations from the Allen dataset the model posterior still had the same ion channel composition as in our initial experiments, except for one trace. However, the posterior predictive performance on the Allen dataset was substantially worse (MSE of $0.48 \pm 0.55$ on the summary statistics, Fig. \ref{fig:app:HH_susmstats}).

When we applied SBC to the Hodgkin-Huxley model we found a c2st mode of 0.53 (with 0.49/0.64 as  .05/.95 percentiles, Fig. \ref{fig:app:HH_SBC}a). While this indicates overall a good calibration, we found parameter specific differences. For example the posteriors for $g_L$ were slightly under confident and for $E_L$ the true parameter was systematically overestimated (Fig. \ref{fig:app:HH_SBC}b). However, given the relatively small training dataset of 100k prior samples, and the complexity of the Hodgkin-Huxley model, the posteriors are relatively well calibrated.

\section{Discussion} \label{sec:discussion}

We presented SBMI, a method for inferring posterior distributions over both model components of scientific simulators and their associated parameters end-to-end. %
For the model inference network, we used a mixture of conditional multivariate binary Grassmann distributions to flexibly and efficiently approximate posterior distributions over models.  %
To deal with the variable dimensionality of the parameter posterior, we used a Gaussian Mixture Density Network which allows efficient marginalization over absent model components during training time. %
By inferring the joint posterior distribution over models and parameters, SBMI allows us to learn parameter dependencies between model components and compensatory mechanisms, in a fully amortized way. 

We first showcased SBMI on additive models and showed that posteriors retrieved by SBMI are in very close agreement with reference posteriors. %
Our application of DMMs yielded posteriors with highly accurate posterior predictives, and allowed identification of compensatory mechanisms for some parameters. 
This demonstrates the  importance of a `model-aware' interpretation of parameter posteriors, enabled by SBMI. 
On experimental data, SBMI automatically retrieved a model which was previously suggested by domain-scientists 
\citep{shinn2020pyddm}, and which outperformed simpler alternatives.
Finally, we ran SBMI on the Hodgkin-Huxley model, both on synthetic data and voltage recordings from the Allen Celltype database. For most of the traces we recover the same model structure as previously used to fit these traces but for some traces additional $Ca_L$ channels might be advantageous. Further in-depth experiments with more channel types could be run to identify the cell mechanisms more exactly.

SBMI gives us not only access to full parameter posteriors, but also infers the uncertainty related to the model \emph{choice} itself and potential interactions between the parameters of different model components. %
For symbolic regression, a similar perspective was presented \cite{werner2022uncertainty} who estimated local uncertainties by Laplace approximations, %
and used a fixed number of equations for the global uncertainty. While this gives some measure of uncertainty, SBMI is able to recover the \emph{full} posterior and its associated uncertainty. 

SBMI enables us to compare different model compositions in a fully amortized manner, allowing scientists to test and compare a large set of competing theories without the need to exhaustively infer each possible combination individually for separate comparisons based on Bayes-factors. Additionally, the amortized nature of SBMI makes it easy to check how robust posteriors over models are when observations change. 
Similarly, amortization also makes it straightforward to perform additional coverage and calibration tests \citep{zhao2021diagnostics,hermans2021averting}. %

For real-world applications the success of SBI relies on appropriate prior choices \citep{oesterle2020bayesian,deistler2022energy} and a well chosen model prior for SBMI can not only increase the data efficiency by simulating more “meaningful” models, but can also enhance the model inference by decreasing the space of possible combinations. 
Although the presented framework already covers many scientific scenarios, the representation of model prior could be further enhanced by lifting the restriction of an ordered model vector of fixed length. %
Using more flexible embedding networks like transformers \citep{vaswani2017attention, lee2019settransformer} could be used to generalize SBMI to simulator outputs $x$ of varying size \citep{biggio2021neural}.

In summary, SBMI provides a powerful tool for data-driven scientific inquiry. It will allow scientists to systematically identify essential model components which are consistent with observed data. Incorporating the uncertainty into their model choices will help to resolve competing models and theories.

\section*{Impact Statement}
We used data recorded from animal experiments in monkey and
mouse. The data we used were recorded in independent experiments and are publicly
available \cite{roitman2002data,Allen2016}.

While our paper presents work whose goal is to advance the field of Machine Learning and scientific discovery, there are many potential societal consequences of our work, none which have unique implications that warrent specifically being highlighted here.

\section*{Acknowledgements}
We thank all group members of the Mackelab for their insightful discussions and valuable feedback on the manuscript. 
This work was funded by the Deutsche Forschungsgemeinschaft (DFG, German Research Foundation) under Germany’s Excellence Strategy – EXC number 2064/1 – 390727645, and under SFB 1233, Robust Vision: Inference Principles and Neural Mechanisms, project 6, number: 276693517 and by the German Federal Ministry of Education and Research (BMBF): Tübingen AI Center, FKZ: 01IS18039A. Co-funded by the European Union (ERC, DeepCoMechTome, 101089288).

\bibliography{library}

\begin{thebibliography}{62}
\providecommand{\natexlab}[1]{#1}
\providecommand{\url}[1]{\texttt{#1}}
\expandafter\ifx\csname urlstyle\endcsname\relax
  \providecommand{\doi}[1]{doi: #1}\else
  \providecommand{\doi}{doi: \begingroup \urlstyle{rm}\Url}\fi

\bibitem[Akiba et~al.(2019)Akiba, Sano, Yanase, Ohta, and Koyama]{optuna_2019}
Akiba, T., Sano, S., Yanase, T., Ohta, T., and Koyama, M.
\newblock Optuna: A next-generation hyperparameter optimization framework.
\newblock In \emph{Proceedings of the 25th {ACM} {SIGKDD} International
  Conference on Knowledge Discovery and Data Mining}, 2019.

\bibitem[{Allen Institute for Brain Science}(2016)]{Allen2016}
{Allen Institute for Brain Science}.
\newblock Allen cell type database, 2016.
\newblock URL \url{http://celltypes.brain-map.org/}.

\bibitem[Arai(2021)]{arai2021grassmann}
Arai, T.
\newblock Multivariate binary probability distribution in the grassmann
  formalism.
\newblock \emph{Physical Review E}, 103\penalty0 (6):\penalty0 062104, 2021.

\bibitem[Bakarji et~al.(2022)Bakarji, Champion, Kutz, and
  Brunton]{bakarji2022discovering}
Bakarji, J., Champion, K., Kutz, J.~N., and Brunton, S.~L.
\newblock Discovering governing equations from partial measurements with deep
  delay autoencoders.
\newblock \emph{arXiv preprint arXiv:2201.05136}, 2022.

\bibitem[Becker et~al.(2022)Becker, Klein, Neitz, Parascandolo, and
  Kilbertus]{becker2022discovering}
Becker, S., Klein, M., Neitz, A., Parascandolo, G., and Kilbertus, N.
\newblock Discovering ordinary differential equations that govern time-series.
\newblock \emph{arXiv preprint arXiv:2211.02830}, 2022.

\bibitem[Biggio et~al.(2021)Biggio, Bendinelli, Neitz, Lucchi, and
  Parascandolo]{biggio2021neural}
Biggio, L., Bendinelli, T., Neitz, A., Lucchi, A., and Parascandolo, G.
\newblock Neural symbolic regression that scales.
\newblock In \emph{International Conference on Machine Learning}, pp.\
  936--945. PMLR, 2021.

\bibitem[Boelts et~al.(2019)Boelts, Lueckmann, Goncalves, Sprekeler, and
  Macke]{boelts2019comparing}
Boelts, J., Lueckmann, J.-M., Goncalves, P.~J., Sprekeler, H., and Macke, J.~H.
\newblock Comparing neural simulations by neural density estimation.
\newblock In \emph{Conference on Cognitive Computational Neuroscience}, pp.\
  1289--1299, 2019.

\bibitem[Boelts et~al.(2022)Boelts, Lueckmann, Gao, and
  Macke]{boelts2022flexible}
Boelts, J., Lueckmann, J.-M., Gao, R., and Macke, J.~H.
\newblock Flexible and efficient simulation-based inference for models of
  decision-making.
\newblock \emph{Elife}, 11:\penalty0 e77220, 2022.

\bibitem[Bradbury et~al.(2018)Bradbury, Frostig, Hawkins, Johnson, Leary,
  Maclaurin, Necula, Paszke, Vander{P}las, Wanderman-{M}ilne, and
  Zhang]{jax2018github}
Bradbury, J., Frostig, R., Hawkins, P., Johnson, M.~J., Leary, C., Maclaurin,
  D., Necula, G., Paszke, A., Vander{P}las, J., Wanderman-{M}ilne, S., and
  Zhang, Q.
\newblock {JAX}: composable transformations of {P}ython+{N}um{P}y programs,
  2018.

\bibitem[Brunton et~al.(2016)Brunton, Proctor, and
  Kutz]{brunton2016discovering}
Brunton, S.~L., Proctor, J.~L., and Kutz, J.~N.
\newblock Discovering governing equations from data by sparse identification of
  nonlinear dynamical systems.
\newblock \emph{Proceedings of the national academy of sciences}, 113\penalty0
  (15):\penalty0 3932--3937, 2016.

\bibitem[Chan et~al.(2018)Chan, Perrone, Spence, Jenkins, Mathieson, and
  Song]{chan2018likelihood}
Chan, J., Perrone, V., Spence, J., Jenkins, P., Mathieson, S., and Song, Y.
\newblock A likelihood-free inference framework for population genetic data
  using exchangeable neural networks.
\newblock \emph{Advances in neural information processing systems}, 31, 2018.

\bibitem[Copernicus(1543)]{copernicus1543revolutionibus}
Copernicus, N.
\newblock De revolutionibus orbium coelestium, 1543.

\bibitem[Cranmer et~al.(2020{\natexlab{a}})Cranmer, Brehmer, and
  Louppe]{cranmer2020frontier}
Cranmer, K., Brehmer, J., and Louppe, G.
\newblock The frontier of simulation-based inference.
\newblock \emph{Proceedings of the National Academy of Sciences}, 117\penalty0
  (48):\penalty0 30055--30062, 2020{\natexlab{a}}.

\bibitem[Cranmer et~al.(2020{\natexlab{b}})Cranmer, Sanchez~Gonzalez,
  Battaglia, Xu, Cranmer, Spergel, and Ho]{cranmer2020discovering}
Cranmer, M., Sanchez~Gonzalez, A., Battaglia, P., Xu, R., Cranmer, K., Spergel,
  D., and Ho, S.
\newblock Discovering symbolic models from deep learning with inductive biases.
\newblock \emph{Advances in Neural Information Processing Systems},
  33:\penalty0 17429--17442, 2020{\natexlab{b}}.

\bibitem[Dax et~al.(2022)Dax, Green, Gair, Deistler, Sch{\"o}lkopf, and
  Macke]{dax2022group}
Dax, M., Green, S.~R., Gair, J., Deistler, M., Sch{\"o}lkopf, B., and Macke,
  J.~H.
\newblock Group equivariant neural posterior estimation.
\newblock In \emph{International Conference on Learning Representations}, 2022.

\bibitem[Deistler et~al.(2022)Deistler, Macke, and
  Gon{\c{c}}alves]{deistler2022energy}
Deistler, M., Macke, J.~H., and Gon{\c{c}}alves, P.~J.
\newblock Energy-efficient network activity from disparate circuit parameters.
\newblock \emph{Proceedings of the National Academy of Sciences}, 119\penalty0
  (44):\penalty0 e2207632119, 2022.

\bibitem[Dub{\v{c}}{\'a}kov{\'a}(2011)]{dubvcakova2011eureqa}
Dub{\v{c}}{\'a}kov{\'a}, R.
\newblock Eureqa: software review, 2011.

\bibitem[Friedman(2003)]{friedman2003multivariate}
Friedman, J.~H.
\newblock On multivariate goodness-of-fit and two-sample testing.
\newblock \emph{Statistical Problems in Particle Physics, Astrophysics, and
  Cosmology}, 1:\penalty0 311, 2003.

\bibitem[Germain et~al.(2015)Germain, Gregor, Murray, and
  Larochelle]{germain2015made}
Germain, M., Gregor, K., Murray, I., and Larochelle, H.
\newblock Made: Masked autoencoder for distribution estimation.
\newblock In \emph{International conference on machine learning}, pp.\
  881--889. PMLR, 2015.

\bibitem[Gon{\c{c}}alves et~al.(2020)Gon{\c{c}}alves, Lueckmann, Deistler,
  Nonnenmacher, {\"O}cal, Bassetto, Chintaluri, Podlaski, Haddad, Vogels,
  et~al.]{gonccalves2020training}
Gon{\c{c}}alves, P.~J., Lueckmann, J.-M., Deistler, M., Nonnenmacher, M.,
  {\"O}cal, K., Bassetto, G., Chintaluri, C., Podlaski, W.~F., Haddad, S.~A.,
  Vogels, T.~P., et~al.
\newblock Training deep neural density estimators to identify mechanistic
  models of neural dynamics.
\newblock \emph{Elife}, 9:\penalty0 e56261, 2020.

\bibitem[Greenberg et~al.(2019)Greenberg, Nonnenmacher, and
  Macke]{greenberg2019automatic}
Greenberg, D., Nonnenmacher, M., and Macke, J.
\newblock Automatic posterior transformation for likelihood-free inference.
\newblock In \emph{International Conference on Machine Learning}, pp.\
  2404--2414. PMLR, 2019.

\bibitem[Groschner et~al.(2022)Groschner, Malis, Zuidinga, and
  Borst]{groschner2022biophysical}
Groschner, L.~N., Malis, J.~G., Zuidinga, B., and Borst, A.
\newblock A biophysical account of multiplication by a single neuron.
\newblock \emph{Nature}, 603\penalty0 (7899):\penalty0 119--123, 2022.

\bibitem[Hagberg et~al.(2008)Hagberg, Schult, and Swart]{hagbert2008networkX}
Hagberg, A.~A., Schult, D.~A., and Swart, P.~J.
\newblock Exploring network structure, dynamics, and function using networkx.
\newblock In Varoquaux, G., Vaught, T., and Millman, J. (eds.),
  \emph{Proceedings of the 7th Python in Science Conference}, pp.\  11 -- 15,
  Pasadena, CA USA, 2008.

\bibitem[Hermans et~al.(2021)Hermans, Delaunoy, Rozet, Wehenkel, and
  Louppe]{hermans2021averting}
Hermans, J., Delaunoy, A., Rozet, F., Wehenkel, A., and Louppe, G.
\newblock Averting a crisis in simulation-based inference.
\newblock \emph{arXiv preprint arXiv:2110.06581}, 2021.

\bibitem[Hethcote(2000)]{hethcote2000SIR}
Hethcote, H.~W.
\newblock The mathematics of infectious diseases.
\newblock \emph{SIAM review}, 42\penalty0 (4):\penalty0 599--653, 2000.

\bibitem[Hodgkin \& Huxley(1952)Hodgkin and Huxley]{hodgkin1952quantitative}
Hodgkin, A.~L. and Huxley, A.~F.
\newblock A quantitative description of membrane current and its application to
  conduction and excitation in nerve.
\newblock \emph{The Journal of physiology}, 117\penalty0 (4):\penalty0 500,
  1952.

\bibitem[Jeffreys(1998)]{jeffreys1998theory}
Jeffreys, H.
\newblock \emph{The theory of probability}.
\newblock OuP Oxford, 1998.

\bibitem[Kass \& Raftery(1995)Kass and Raftery]{kass1995bayes}
Kass, R.~E. and Raftery, A.~E.
\newblock Bayes factors.
\newblock \emph{Journal of the american statistical association}, 90\penalty0
  (430):\penalty0 773--795, 1995.

\bibitem[Kermack \& McKendrick(1927)Kermack and McKendrick]{kermack1927SIR}
Kermack, W.~O. and McKendrick, A.~G.
\newblock A contribution to the mathematical theory of epidemics.
\newblock \emph{Proceedings of the royal society of london. Series A,
  Containing papers of a mathematical and physical character}, 115\penalty0
  (772):\penalty0 700--721, 1927.

\bibitem[Latimer et~al.(2015)Latimer, Yates, Meister, Huk, and
  Pillow]{latimer2015single}
Latimer, K.~W., Yates, J.~L., Meister, M.~L., Huk, A.~C., and Pillow, J.~W.
\newblock Single-trial spike trains in parietal cortex reveal discrete steps
  during decision-making.
\newblock \emph{Science}, 349\penalty0 (6244):\penalty0 184--187, 2015.

\bibitem[Lee et~al.(2019)Lee, Lee, Kim, Kosiorek, Choi, and
  Teh]{lee2019settransformer}
Lee, J., Lee, Y., Kim, J., Kosiorek, A., Choi, S., and Teh, Y.~W.
\newblock Set transformer: A framework for attention-based
  permutation-invariant neural networks.
\newblock In \emph{International conference on machine learning}, pp.\
  3744--3753. PMLR, 2019.

\bibitem[Lueckmann et~al.(2017)Lueckmann, Goncalves, Bassetto, {\"O}cal,
  Nonnenmacher, and Macke]{lueckmann2017flexible}
Lueckmann, J.-M., Goncalves, P.~J., Bassetto, G., {\"O}cal, K., Nonnenmacher,
  M., and Macke, J.~H.
\newblock Flexible statistical inference for mechanistic models of neural
  dynamics.
\newblock \emph{Advances in neural information processing systems}, 30, 2017.

\bibitem[Mancini et~al.(2022)Mancini, Docherty, Price, and
  McEwen]{mancini2022bayesian}
Mancini, A.~S., Docherty, M., Price, M., and McEwen, J.
\newblock Bayesian model comparison for simulation-based inference.
\newblock \emph{arXiv preprint arXiv:2207.04037}, 2022.

\bibitem[Marlier et~al.(2021)Marlier, Br{\"u}ls, and
  Louppe]{marlier2021simulationrobotics}
Marlier, N., Br{\"u}ls, O., and Louppe, G.
\newblock Simulation-based bayesian inference for multi-fingered robotic
  grasping.
\newblock \emph{arXiv preprint arXiv:2109.14275}, 2021.

\bibitem[Martius \& Lampert(2016)Martius and Lampert]{martius2016extrapolation}
Martius, G. and Lampert, C.~H.
\newblock Extrapolation and learning equations.
\newblock \emph{arXiv preprint arXiv:1610.02995}, 2016.

\bibitem[McCormick \& Huguenard(1992)McCormick and
  Huguenard]{mccormick1992model}
McCormick, D.~A. and Huguenard, J.~R.
\newblock A model of the electrophysiological properties of thalamocortical
  relay neurons.
\newblock \emph{Journal of neurophysiology}, 68\penalty0 (4):\penalty0
  1384--1400, 1992.

\bibitem[Meurer et~al.(2017)Meurer, Smith, Paprocki, {\v{C}}ert{\'\i}k,
  Kirpichev, Rocklin, Kumar, Ivanov, Moore, Singh, et~al.]{meurer2017sympy}
Meurer, A., Smith, C.~P., Paprocki, M., {\v{C}}ert{\'\i}k, O., Kirpichev,
  S.~B., Rocklin, M., Kumar, A., Ivanov, S., Moore, J.~K., Singh, S., et~al.
\newblock Sympy: symbolic computing in python.
\newblock \emph{PeerJ Computer Science}, 3:\penalty0 e103, 2017.

\bibitem[Oesterle et~al.(2020)Oesterle, Behrens, Schr{\"o}der, Hermann, Euler,
  Franke, Smith, Zeck, and Berens]{oesterle2020bayesian}
Oesterle, J., Behrens, C., Schr{\"o}der, C., Hermann, T., Euler, T., Franke,
  K., Smith, R.~G., Zeck, G., and Berens, P.
\newblock Bayesian inference for biophysical neuron models enables stimulus
  optimization for retinal neuroprosthetics.
\newblock \emph{Elife}, 9:\penalty0 e54997, 2020.

\bibitem[Papamakarios \& Murray(2016)Papamakarios and
  Murray]{papamakarios2016fast}
Papamakarios, G. and Murray, I.
\newblock Fast $\varepsilon$-free inference of simulation models with bayesian
  conditional density estimation.
\newblock \emph{Advances in neural information processing systems}, 29, 2016.

\bibitem[Papamakarios et~al.(2021)Papamakarios, Nalisnick, Rezende, Mohamed,
  and Lakshminarayanan]{papamakarios2021normalizing}
Papamakarios, G., Nalisnick, E., Rezende, D.~J., Mohamed, S., and
  Lakshminarayanan, B.
\newblock Normalizing flows for probabilistic modeling and inference.
\newblock \emph{Journal of Machine Learning Research}, 22\penalty0
  (57):\penalty0 1--64, 2021.

\bibitem[Paszke et~al.(2019)Paszke, Gross, Massa, Lerer, Bradbury, Chanan,
  Killeen, Lin, Gimelshein, Antiga, et~al.]{paszke2019pytorch}
Paszke, A., Gross, S., Massa, F., Lerer, A., Bradbury, J., Chanan, G., Killeen,
  T., Lin, Z., Gimelshein, N., Antiga, L., et~al.
\newblock Pytorch: An imperative style, high-performance deep learning library.
\newblock \emph{Advances in neural information processing systems}, 32, 2019.

\bibitem[Podlaski et~al.(2017)Podlaski, Seeholzer, Groschner, Miesenb{\"o}ck,
  Ranjan, and Vogels]{podlaski2017mapping}
Podlaski, W.~F., Seeholzer, A., Groschner, L.~N., Miesenb{\"o}ck, G., Ranjan,
  R., and Vogels, T.~P.
\newblock Mapping the function of neuronal ion channels in model and
  experiment.
\newblock \emph{Elife}, 6:\penalty0 e22152, 2017.

\bibitem[Pospischil et~al.(2008)Pospischil, Toledo-Rodriguez, Monier,
  Piwkowska, Bal, Fr{\'e}gnac, Markram, and Destexhe]{pospischil2008minimal}
Pospischil, M., Toledo-Rodriguez, M., Monier, C., Piwkowska, Z., Bal, T.,
  Fr{\'e}gnac, Y., Markram, H., and Destexhe, A.
\newblock Minimal hodgkin--huxley type models for different classes of cortical
  and thalamic neurons.
\newblock \emph{Biological cybernetics}, 99:\penalty0 427--441, 2008.

\bibitem[Radev et~al.(2020)Radev, Mertens, Voss, Ardizzone, and
  K{\"o}the]{radev2020bayesflow}
Radev, S.~T., Mertens, U.~K., Voss, A., Ardizzone, L., and K{\"o}the, U.
\newblock Bayesflow: Learning complex stochastic models with invertible neural
  networks.
\newblock \emph{IEEE transactions on neural networks and learning systems},
  33\penalty0 (4):\penalty0 1452--1466, 2020.

\bibitem[Radev et~al.(2021)Radev, D'Alessandro, Mertens, Voss, K{\"o}the, and
  B{\"u}rkner]{radev2021amortizedmodelcomparison}
Radev, S.~T., D'Alessandro, M., Mertens, U.~K., Voss, A., K{\"o}the, U., and
  B{\"u}rkner, P.-C.
\newblock Amortized bayesian model comparison with evidential deep learning.
\newblock \emph{IEEE Transactions on Neural Networks and Learning Systems},
  2021.

\bibitem[Ratcliff(1978)]{ratcliff1978theory}
Ratcliff, R.
\newblock A theory of memory retrieval.
\newblock \emph{Psychological review}, 85\penalty0 (2):\penalty0 59, 1978.

\bibitem[Ratcliff \& McKoon(2008)Ratcliff and
  McKoon]{ratcliff2008diffusion_review}
Ratcliff, R. and McKoon, G.
\newblock The diffusion decision model: theory and data for two-choice decision
  tasks.
\newblock \emph{Neural computation}, 20\penalty0 (4):\penalty0 873--922, 2008.

\bibitem[Roitman \& Shadlen(2002)Roitman and Shadlen]{roitman2002data}
Roitman, J.~D. and Shadlen, M.~N.
\newblock Response of neurons in the lateral intraparietal area during a
  combined visual discrimination reaction time task.
\newblock \emph{Journal of neuroscience}, 22\penalty0 (21):\penalty0
  9475--9489, 2002.

\bibitem[Sahoo et~al.(2018)Sahoo, Lampert, and Martius]{sahoo18a}
Sahoo, S., Lampert, C., and Martius, G.
\newblock Learning equations for extrapolation and control.
\newblock In Dy, J. and Krause, A. (eds.), \emph{Proceedings of the 35th
  International Conference on Machine Learning}, volume~80 of \emph{Proceedings
  of Machine Learning Research}, pp.\  4442--4450. PMLR, 10--15 Jul 2018.

\bibitem[Scala et~al.(2021)Scala, Kobak, Bernabucci, Bernaerts, Cadwell,
  Castro, Hartmanis, Jiang, Laturnus, Miranda, et~al.]{scala2021phenotypic}
Scala, F., Kobak, D., Bernabucci, M., Bernaerts, Y., Cadwell, C.~R., Castro,
  J.~R., Hartmanis, L., Jiang, X., Laturnus, S., Miranda, E., et~al.
\newblock Phenotypic variation of transcriptomic cell types in mouse motor
  cortex.
\newblock \emph{Nature}, 598\penalty0 (7879):\penalty0 144--150, 2021.

\bibitem[Schmidt \& Lipson(2009)Schmidt and
  Lipson]{schmidt2009distilling_eureqa}
Schmidt, M. and Lipson, H.
\newblock Distilling free-form natural laws from experimental data.
\newblock \emph{Science}, 324\penalty0 (5923):\penalty0 81--85, 2009.

\bibitem[Shinn et~al.(2020)Shinn, Lam, and Murray]{shinn2020pyddm}
Shinn, M., Lam, N.~H., and Murray, J.~D.
\newblock A flexible framework for simulating and fitting generalized
  drift-diffusion models.
\newblock \emph{ELife}, 9:\penalty0 e56938, 2020.

\bibitem[Sisson et~al.(2018)Sisson, Fan, and Beaumont]{sisson2018handbook}
Sisson, S.~A., Fan, Y., and Beaumont, M.
\newblock \emph{Handbook of approximate Bayesian computation}.
\newblock CRC Press, 2018.

\bibitem[Talts et~al.(2018)Talts, Betancourt, Simpson, Vehtari, and
  Gelman]{talts2018validating}
Talts, S., Betancourt, M., Simpson, D., Vehtari, A., and Gelman, A.
\newblock Validating bayesian inference algorithms with simulation-based
  calibration.
\newblock \emph{arXiv preprint arXiv:1804.06788}, 2018.

\bibitem[Tejero-Cantero et~al.(2020)Tejero-Cantero, Boelts, Deistler,
  Lueckmann, Durkan, Gonçalves, Greenberg, and Macke]{tejerocantero2020sbi}
Tejero-Cantero, A., Boelts, J., Deistler, M., Lueckmann, J.-M., Durkan, C.,
  Gonçalves, P.~J., Greenberg, D.~S., and Macke, J.~H.
\newblock sbi: A toolkit for simulation-based inference.
\newblock \emph{Journal of Open Source Software}, 5\penalty0 (52):\penalty0
  2505, 2020.

\bibitem[Trotta(2008)]{trotta2008bayessky}
Trotta, R.
\newblock Bayes in the sky: Bayesian inference and model selection in
  cosmology.
\newblock \emph{Contemporary Physics}, 49\penalty0 (2):\penalty0 71--104, 2008.

\bibitem[Turner et~al.(2015)Turner, Van~Maanen, and
  Forstmann]{turner2015informing}
Turner, B.~M., Van~Maanen, L., and Forstmann, B.~U.
\newblock Informing cognitive abstractions through neuroimaging: the neural
  drift diffusion model.
\newblock \emph{Psychological review}, 122\penalty0 (2):\penalty0 312, 2015.

\bibitem[Vaswani et~al.(2017)Vaswani, Shazeer, Parmar, Uszkoreit, Jones, Gomez,
  Kaiser, and Polosukhin]{vaswani2017attention}
Vaswani, A., Shazeer, N., Parmar, N., Uszkoreit, J., Jones, L., Gomez, A.~N.,
  Kaiser, {\L}., and Polosukhin, I.
\newblock Attention is all you need.
\newblock \emph{Advances in neural information processing systems}, 30, 2017.

\bibitem[Werner et~al.(2021)Werner, Junginger, Hennig, and
  Martius]{werner2021informed}
Werner, M., Junginger, A., Hennig, P., and Martius, G.
\newblock Informed equation learning.
\newblock \emph{arXiv preprint arXiv:2105.06331}, 2021.

\bibitem[Werner et~al.(2022)Werner, Junginger, Hennig, and
  Martius]{werner2022uncertainty}
Werner, M., Junginger, A., Hennig, P., and Martius, G.
\newblock Uncertainty in equation learning.
\newblock In \emph{Proceedings of the Genetic and Evolutionary Computation
  Conference Companion}, pp.\  2298--2305, 2022.

\bibitem[Yadan(2019)]{Yadan2019Hydra}
Yadan, O.
\newblock Hydra - a framework for elegantly configuring complex applications.
\newblock Github, 2019.
\newblock URL \url{https://github.com/facebookresearch/hydra}.

\bibitem[Zhao et~al.(2021)Zhao, Dalmasso, Izbicki, and
  Lee]{zhao2021diagnostics}
Zhao, D., Dalmasso, N., Izbicki, R., and Lee, A.~B.
\newblock Diagnostics for conditional density models and bayesian inference
  algorithms.
\newblock In \emph{Uncertainty in Artificial Intelligence}, pp.\  1830--1840.
  PMLR, 2021.

\end{thebibliography}
\bibliographystyle{icml2024}

\newpage
\appendix
\onecolumn
\normalsize
\newpage
\setcounter{figure}{0}
\renewcommand{\thefigure}{S\arabic{figure}}
\setcounter{table}{0}
\renewcommand{\thetable}{S\arabic{table}}
\setcounter{section}{0}
\renewcommand{\thesection}{A\arabic{section}}

\section*{\LARGE Appendix}

\section{Software and Computational Ressources}

Code is available at \url{https://github.com/mackelab/simulation_based_model_inference}.

All networks were implemented in \emph{pytorch} \cite{paszke2019pytorch}.
Additionally, we used the following software and toolboxes in this work: \emph{sbi} \cite{tejerocantero2020sbi} for the implementation of SBMI, \emph{NetworkX} \cite{hagbert2008networkX} for the construction of prior graphs, \emph{SymPy} \cite{meurer2017sympy} for symbolic calculations, \emph{pyDDM} \cite{shinn2020pyddm} as the backend for the DDM experiments. 
To manage the configuration settings we used \emph{Hydra} \cite{Yadan2019Hydra} and the \emph{Optuna Sweeper} \cite{optuna_2019} plugin for a coarse hyperparameter search in the DDM setting.

All models were trained on an Nvidia RTX 2080ti GPU accessed via a slurm cluster.

\section{Mixture of Grassmann Distribution} \label{app:grassmann}

Previously, Arai introduced the Grassmann formalism for multivariate binary distributions \cite{arai2021grassmann} by using anticommuting numbers, called Grassmann numbers. A Grassmann distribution $\mathcal{G}$ is an $n$-dimensional binary distribution parameterized by an $n\times n$ matrix $\Sigma$. The probability mass function of $\mathcal{G}$ with parameter $\Sigma$ on $Y=(Y_1,...,Y_n)$  is defined as 
\begin{equation*}
    \mathcal{G}(y|\Sigma) = \det \begin{pmatrix}
                            \Sigma_{11}^{y_1} (1-\Sigma_{11})^{1-y_1} & \Sigma_{12} (-1)^{1-y_2} & \cdots\\
                            \Sigma_{21} (-1)^{1-y_1} &\Sigma_{22}^{y_2} (1-\Sigma_{22})^{1-y_2} &  \cdots\\
                            \vdots & \vdots & \ddots
                        \end{pmatrix}.
\end{equation*}
For a valid distribution $\Sigma^{-1}-{I}$ must be a $P_0$ matrix, but has otherwise no further constraints \cite{arai2021grassmann}.

This definition gives access to the analytical derivations of properties such as the mean, covariance, and marginal and conditional distributions. Here, we only recapitulate the analytical formula for conditional distribution, which is used for sampling. Their derivation and further details can be found in \cite{arai2021grassmann}.
In the following paragraph we follow the notation of Arai \cite{arai2021grassmann}.

For a conditional distribution on $Y=(Y_1,...,Y_n)$, we denote by $C$ the indices of the observed variables $y_j \in \{0,1\}$ and $R$ the remaining indices $R = \{1,...,n\}\setminus C$. Without loss of generality, the parameter matrix can be written as 
\begin{equation*}
    \Sigma = \begin{pmatrix}
                \Sigma_{RR} & \Sigma_{RC} \\
                \Sigma_{CR} & \Sigma_{CC}
            \end{pmatrix}.
\end{equation*}
The conditional distribution is then given by the Grassmann distribution
\begin{equation*}
    p(y_R | y_C) = \mathcal{G}(y_R | \Sigma_{R|y_C}) 
\end{equation*}
with 
\begin{equation*}
   \Sigma_{R|y_C} = \Sigma_{RR} - \Sigma_{RC}\big(\Sigma_{CC}-\text{diag}(1-y_C)\big)^{-1}\Sigma_{CR},
\end{equation*}
where $\text{diag}(1-y_C)$ is the diagonal matrix with $(1-y_C)$ on its diagonal. 
An analogous formula can be derived by using the notation $\Lambda^{-1} = \Sigma$ \cite{arai2021grassmann}.

\paragraph{Mixture of Grassmann Distribution}
We define a mixture of Grassmann distribution (MoGr) on $\{0,1\}^n$ in the same formalism as $p(y) =\sum_i \alpha_i \mathcal{G}_i(y|\Sigma_i)$ for a finite partition $\sum_i \alpha_i = 1$ and Grassmann distributions $\mathcal{G}_i$. Using the means $\mu_i$ and covariances $C_i$ for each component $\mathcal{G}_i$ we can calculate the mean and covariance for the mixture distribution by introducing a discrete latent variable $Z$ and reformulate the mixture distribution as
\begin{align*}
    p(y|Z=i)&=\mathcal{G}_i(y|\Sigma_i), \\
    p(Z=i) &= \alpha_i.
\end{align*}
Using the law of total expectation and variance we get analytical expressions for the mean and covariance of a MoGr: 
\begin{equation*}
    \E[Y]=\E[\E[Y|Z=i]] = \sum_{i} \alpha_i \mu_i
\end{equation*}
and 
\begin{align*}
    \Cov(Y) &= \E[\Cov(Y|Z=i)]+\Cov(\E[Y|Z=i]) \\ 
            &=\sum_{i} \alpha_i C_i + \sum_{i} \alpha_i (\mu_i-\bar{\mu})(\mu_i-\bar{\mu})^T,
\end{align*}
where $\bar{\mu}=\E[Y]$.

To sample from a MoGr we use the standard procedure of first sampling one component $z_i \sim p(Z)$, and then using the conditional expression of a Grassmann distribution to sample $y\sim \mathcal{G}_{z_i}$.

\paragraph{Implementation}
Arai \cite{arai2021grassmann} proposed the following parametrization for $\Sigma$ that ensures the $P_0$ criterion for $\Sigma$:
\begin{equation*}
    \Sigma^{-1} = B C ^{-1} + I,
\end{equation*}
where $B$ and $C$ are strictly row diagonal dominant matrices, namely
\begin{equation*}
    b_{ii} > \sum_{j\neq i} |b_{ij}|, \quad \text{and} \quad c_{ii} > \sum_{j\neq i} |c_{ij}|.
\end{equation*}
We make use of this parametrization by optimizing unconstrained matrices $\tilde{B}$ and $\tilde{C}$ and defining $B$ by replacing the diagonal elements of $\tilde{B}$ by
\begin{equation*}
        b_{ii} = \exp(\tilde{b}_{ii}) + \sum_{j\neq i} |\tilde{b}_{ij}|,
\end{equation*}
and analogously for $C$. Instead of $\exp$ any other positive function could be chosen and even the non-negative ReLU function showed good training behaviour in initial experiments.

We used a similar parameterization for a mixture of Grassmann distribution for each component and a softmax layer to learn the partition $\sum_i \alpha_i =1$.

\section{Model Prior} \label{app:sec:model_prior}

\paragraph{Sampling from the prior:}
A draw from the prior corresponds to a random walk on the graph  $(\mathcal{M}, E, \mathcal{R})$, starting at node $M_0$ and walking through the graph in the following way (see Algorithm \ref{alg:prior} for pseudocode): For each step, we first normalize the outgoing edges for the current node $M_c$. We use these weights to sample the next node $M_{c+1}$ from a categorical distribution on all connected nodes and append it to the set of sampled nodes $S$. Next, we update all weights following the updating rules $\mathcal{R}$. We then start the next step by normalizing the outgoing edges of the updated current node and repeat this procedure until the end node $M_{N+1}$ is reached. 

\SetKwComment{Comment}{\# }{}
\begin{algorithm}[H]
 \textbf{Inputs:} Directed graph with dynamically changing weights: $(\mathcal{M},\mathcal{E}, \mathcal{R})$, \\
 with nodes $\mathcal{M}=\{M_i\}_{i \in 0,...,N+1}$, edges $\mathcal{E} = \{e_{ij}=(M_i, M_j,w_{ij}) | M_i, M_j \in \mathcal{M}, w_{ij} \in \mathbb{R}_{\geq 0} \}$, and updating rules   $\mathcal{R}=\{R_k|R_k:(\mathcal{E},S) \to \mathcal{E},~ S \subset \mathcal{M}\}_{k \in K}$  \\
 \textbf{Outputs:} Prior sample $\Tilde{M} \sim p(M)$.\\
$S \leftarrow \{M_0\}$ \Comment*[r]{initialize set of sampled nodes}
$M_c \leftarrow M_0$ \Comment*[r]{initialize current node}
\While{$M_{N+1}$ $\notin S$}{
 $\Tilde{w}_{ci} \leftarrow w_{ci}/\sum_i w_{ci}$   \Comment*[r]{normalize outgoing edges}
 $q_c \leftarrow \text{Cat} (\Tilde{w}_{ci}, \{M_i | \Tilde{w}_{ci}>0 \})$  \Comment*[r]{define categorical distribution}
 $M_{c+1} \sim q_c$ \Comment*[r]{sample next node}
 $S \leftarrow S \cup \{M_{c+1}\}$ ;\\
 \For{$R_k \in \mathcal{R}$}{ 
    $\mathcal{E} \leftarrow R_k(\mathcal{E},S)$   \Comment*[r]{update weights}
 }
 $M_c \leftarrow M_{c+1}$ ;
}
$\Tilde{M} = [1 \text{ if } M_i \in S \text{ else } 0, \text{ for } i \in 1,...,N ] $ \Comment*[r]{convert to binary vector of dimension N}
\textbf{return} $\Tilde{M}$\\
\caption{Sampling procedure for model prior}\label{alg:prior}
\end{algorithm}

\paragraph{Example rules $\mathcal{R}$:}
\begin{enumerate}
    \item Example rule $R^1_{x}$ to ensure a conditionally acyclic graph: For a sampled $M_x$ we set all ingoing edges of  $M_x$ to zero:
    \begin{equation*}
    R^1_{x} = \{ \text{if } M_x \in S:  w_{ix} = 0 ~ \forall ~ i = 0,...,N+1\}
    \end{equation*} 
    \item Example rules $R^2_{xy}$ and $R^2_{yx}$ to avoid the co-occurrence of specific components $M_x$ and $M_y$: We set all ingoing edges to $M_y$ to zero if $M_x$ was already sampled and the other way round:
    \begin{align*}
    &R^2_{xy} = \{ \text{if } M_x \in S:  w_{iy} = 0 ~ \forall ~ i = 0,...,N+1\} , \text{ ~and} \\
    &R^2_{yx} = \{ \text{if } M_y \in S:  w_{ix} = 0 ~ \forall ~ i = 0,...,N+1\}.    
    \end{align*}
    \item Example rules $R^3_{xy}$ and $R^3_{yx}$ to decrease the co-occurrence of specific components $M_x$ and $M_y$ by decreasing the weight of all ingoing edges to $M_y$ by a constant factor $c$ if $M_x$ was already sampled and the other way round:
    \begin{align*}
    &R^3_{xy} = \{ \text{if } M_x \in S:  w_{iy} = c w_{iy}  ~ \forall ~ i = 0,...,N+1 \text{ and } 0<c<1\} , \text{ ~and} \\
    &R^3_{yx} = \{ \text{if } M_y \in S:  w_{ix} = c w_{ix} ~ \forall ~ i = 0,...,N+1 \text{ and } 0<c<1 \}.    
    \end{align*}
    \item Example rule $R^4_{x\text{end}}$ to favor simpler models. For a sampled $M_x$ with no direct vertex to the end node, we increase the weights vertices which are directly connected to the end node by a constant factor $c$:
    \begin{equation*}
     R^4_{x\text{end}} = \{ \text{if } M_x \in S \text{ and }  w_{x\text{end}} = 0:  w_{i\text{end}} = c  w_{i\text{end}} ~ \forall ~ i = 0,...,N+1 \text{ and } c > 1\}.
    \end{equation*} 
    
\end{enumerate}

\section{Inference} \label{app:inference}

\SetKwComment{Comment}{\# }{}
\begin{algorithm}[H]
 \textbf{Inputs:} Model prior $p(M)$, parameter priors $p(\theta|M)$, compiler $C$, number of simulations $L$, embedding net $e_\zeta(x)$, model posterior network $q_\psi(M|e)$, parameter posterior network $q_\phi(\theta|M,e)$.\\
 \textbf{Outputs:} Trained embedding network $e_\zeta(x)$, model posterior network $q_\psi(M|x)$, parameter posterior network $q_\phi(\theta|M,x)$.\\
\textbf{Generate dataset:}\\
~\
    \For{$l=1,...,L$}{
     $M_l \sim p(M)$   \Comment*[r]{sample model}
     $\theta_l \sim p(\theta|M_l)$ \Comment*[r]{sample parameters}
     $S_l \leftarrow C(M_l,\theta_l)$ \Comment*[r]{compile simulator}
     $x_l \sim S_l$ \Comment*[r]{simulate data}
    }
    \textbf{return} $\{(M_l,\theta_l,x_l)\}_{l=1,...,L}$\\
\textbf{Training:} \Comment*[r]{We omit the use of training batches here.}
~\
    \While{not converged}{
        $\mathcal{L}^M \leftarrow  - \frac{1}{L} \sum_l\log ~q_\psi(M_l|e_\zeta(x_l))$  \Comment*[r]{compute model loss}
        $\mathcal{L}^\theta \leftarrow  - \frac{1}{L} \sum_l\log ~q_\phi(\theta_l|M_l,e_\zeta(x_l)) $ \Comment*[r]{compute parameter loss}
        $(\zeta, \psi,\phi) \leftarrow (\zeta, \psi,\phi) - \text{Adam}(\nabla_{(\zeta, \psi,\phi)} (\mathcal{L}^M + \mathcal{L}^\theta))$ \Comment*[r]{take gradient step}
    }
    \textbf{return}  $e_\zeta(x)$,  $q_\psi(M|x)$, $q_\phi(\theta|M,x)$\\
    \caption{Simulation-base model inference: SBMI}\label{alg:SBMI}
\end{algorithm}

\subsection{Model Posterior Network} \label{app:model_posteriror_network}
We used a conditional MoGr distribution as model posterior network. 
The conditional parameters $\Sigma_i|x$ are parameterized by two matrices $B_i$ and $C_i$ (Section \ref{app:grassmann}). We used a fully connected neural network with ReLU activation to parametrize the unconstrained matrices $\tilde B_i$, $\tilde C_i$ and a softmax layer for the partition $\alpha$ with $\sum_i \alpha_i =1$.
The input to the MoGr network is the output $e(x)$ of the embedding net and the used hyperparameters can be found in Table \ref{tab:additive_network_details}  and \ref{tab:ddm_network_details}.

\subsection{Parameter Posterior Network} \label{app:parameter_posteriror_network}
For the paremeter posterior network, we used a conditioned mixture of (Gaussian) density network, which allowed us to marginalize analytically over the parameters of the absent model components during training time. 
For efficient training, we divided each batch into sub-batches with the same number of parameters and processed each sub-batch in parallel. 

The conditioning network was implemented as fully connected network with ReLU activation. The specifics for the different settings can be found in Table \ref{tab:additive_network_details},  
\ref{tab:additive_network_detailsBIG}, \ref{tab:ddm_network_details}, and \ref{tab:HH_network_details}.

\subsection{Computational Efficiency}
The parameter posterior network shares the computational complexity with standard MDNs, as marginalization of the MDN is performed analytically. The evaluation of the MoGr distribution involves computing a determinant of a $n \times n$ matrix, which in general has the complexity of $O(n^3)$. However, $n$ is the number of model components, which is relatively small in the presented work (up to 20 for our experiments).

\subsection{Training}
We used the standard training loop of the \emph{sbi} toolbox \cite{tejerocantero2020sbi}: as validation set, we used 10\% of the training samples and as stopping criterion we defined 25 consecutive epochs of no improvement of the loss function on the validation set. 

For the additive model we used a batch size of 3000 samples, for the DDM a batchsize of 2000 samples, and for the Hodgkin-Huxley model of 4000 samples.

\section{Performance Measures}
\label{app:performance measures}
\subsection{MAP Estimate} \label{app:MAP}
Once we trained the full network, we can easily get a \emph{maximum a posteriori} estimate (MAP) by searching the discrete model space: 
\begin{align*}
    \max_{M, \theta } p(M, \theta| x_o ) & = \max_{M, \theta} p(M|x_o) \cdot p(\theta | M, x_o)  \\
                            & = \max_{i \in I } \{ p(M_i|x_o) \cdot \max_\theta p(\theta | M_i, x_o)  \}.   
\end{align*} 

While mathematically correct, this MAP is often dominated by the density function of the parameter posterior, which can take arbitrarily large values for small variances and can be susceptible to noise in the training process. The discrete distribution $p(M|x_o)$ is, however, bounded in $[0,1]$. Therefore, we are often interested in the more stable MAP parameter estimates  of the most likely model: \[\theta_{map}=\argmax_{ \theta } p( \theta| M_{map}, x_o ),\]
where $M_{map} = \argmax\limits_{M_i, i \in I } p(M_i|x_o)$.

This MAP of the most likely model is shown as $f_{MAP}$ in Figure \ref{fig:additive_model}d.

\subsection{Additive Model}
For the additive model we can approximate the ground truth model posterior $p(M|x_o)$ by calculating the model evidence by
\begin{equation*}
    p(M|x_o) = \frac{p(x_o|M) p(M)}{p(x_o)}  \sim p(x_o|M) p(M).
\end{equation*}
The prior $p(M)$ is only given implicitly, but as the model space is low-dimensionial, we can approximate the prior by the empirical sampling distribution $\hat{p}(M)$ (shown in Fig. \ref{fig:additive_model}).
We therefore get the approximation
\begin{align*}
    p(M|x_o) & \sim {p}(M) \int p(x_o|M,\theta)p(\theta|M) d\theta \\
        & \approx \hat{p}(M) \frac{1}{N} \sum_{j=1}^N p(x_o|M,\theta_j), 
\end{align*}
where $\theta_j$ are samples from the parameter prior $p(\theta|M)$. Since we used a Gaussian noise model, we can calculate the expression $p(x_o|M,\theta_j)$ by evaluating $\mathcal{N}(f_{\theta_j}(t), \Sigma_{\theta_j}(t))$.

In practice, we apply importance sampling to avoid regions with a low probability, such that we get
\begin{equation*}
    p(M|x_o)  \sim \hat{p}(M) \frac{1}{N} \sum_{j=1}^N p(x_o|M,\theta_j) \frac{p(\theta_j|M)}{q_\phi (\theta_j| M, x_o)},
\end{equation*}
where $\theta_j \sim q_\phi (\theta| M, x_o)$ are samples from the approximated parameter posterior.  

Even with importance sampling, a lot of samples were necessary to get reliable estimates.  We used 100k samples $\theta_j \sim q_\phi (\theta| M, x_o)$ per observation and were therefore restricted to few observations $x_o$ (100 for the presented results in Section \ref{subsec:additive_model}). 

\subsection{DDM}
We used the mean-squared error (MSE) on the weighted density functions of the two different decisions, similar to \cite{shinn2020pyddm}. In the same work, they showed that the relative MSE is in good correspondence with other performance metrics on the used experimental data. We therefore used the loss function implemented as \texttt{LossSquaredError} in the \emph{pyDDM} package \cite{shinn2020pyddm}.

\section{SBMI Loss and Kullback-Leibler Divergence }

\begin{proposition} \label{proof:KL} Optimizing the SBMI loss function
    $\mathcal{L}(\psi, \phi) = - \frac{1}{L} \sum_{l} \mathcal{L}_{M_l}(\psi) + \mathcal{L}_{\theta_l}(\phi)$ minimizes the expected Kullback-Leibler divergence between the true joint posterior  $p(M,\theta|x)$ and the approximation $q_\phi(M|x)q_\psi(\theta|M,x)$:
 \begin{equation*}
     \mathbb{E}_{p(x)} \big [D_{KL}( p(M,\theta|x)|| q_\phi(M|x)q_\psi(\theta|M,x)\big ].
 \end{equation*}
    
\end{proposition}
    
\begin{proof}
    \begin{align*}
    &\mathbb{E}_{p(x)}[D_{KL}\left( p(M,\theta|x)|| q_\phi(M|x)q_\psi(\theta|M,x) \right)]\\
    & = \mathbb{E}_{p(x)}\left [\mathbb{E}_{p(M,\theta|x)}\left[\log  \frac{p(M,\theta|x)}{q_\phi(M|x)  q_\psi(\theta|M,x)}  \right]\right] \\
    & = \mathbb{E}_{p(x,M,\theta)}\left[ -\log  q_\phi(M|x)  -q_\psi(\theta|M,x)  \right] + C\\
    & \approx \frac{1}{L} \sum_{l} \big( -\log  q_\phi(M|x)  -q_\psi(\theta|M,x)  \big)+ C  \\
    & = - \frac{1}{L} \sum_{l} \mathcal{L}_{M_l}(\psi) + \mathcal{L}_{\theta_l}(\phi) + C\\
    & = \mathcal{L}(\psi, \phi) +C
\end{align*}
where $C$ is a constant independent of $\phi$ and $\psi$.
\end{proof}

\section{Model Details} \label{app:model_details}

\subsection{Additive Model}
\label{app:model_details_additive}

\begin{table}[h]
\caption{ \textbf{Details for the additive model with six components.} The parameter $\theta_1$ in the noise terms $n_1$ and $n_2$ defines the standard deviation of a normal distribution $\mathcal{N}$, and $\mathcal{U}(a,b)$ defines a uniform distribution on the interval $[a,b]$.  Overall the model has seven parameters.
For the performance we report the mean and standard deviation.}
          \label{tab:additive_model_2}
      \centering
        \begin{tabular}{lclccc}
            \toprule
            Model Component     &   Token    & Parameter Prior  & \makecell[c]{Performance \\ $\hat p_{\text{reference}}(M_i|x_o)$}& \makecell[c]{Performance \\  $q_\psi(M_i|x_o)$}    \\
            \midrule
            $\theta_1 \cdot t $ & $l_1$ & $\theta_1 \sim \mathcal{U}(-2,2) $ &0.70 (0.27) & 0.65  (0.24) \\ %
            $\theta_1 \cdot t $  &  $l_2$  & $\theta_1 \sim \mathcal{U}(-2,2) $   &0.70 (0.26) & 0.67 (0.24) \\ %
            $\theta_1 \cdot t^2 $  & $q$   & $\theta_1 \sim \mathcal{U}(-0.5,0.5) $  & 0.97 (0.09) & 0.93 (0.15) \\ %
            $\theta_1 \cdot \sin(\theta_2 t ) $ & $sin$ &\makecell[l]{$\theta_1 \sim \mathcal{U}(0,5)$ \\$\theta_2 \sim \mathcal{U}(0.5,5) $}  & 0.95 (0.15) & 0.91 (0.18) \\ %
            noise$_1$: $n_{t_i} \sim \mathcal{N}(0,\theta_1)$ & $n_1$ & $\theta_1 \sim \mathcal{U}(0.1,2) $ & 1.00 (0.00) & 1.00 (0.00) \\ %
            noise$_2$: $n_{t_i} \sim (t_i+1) \mathcal{N}(0,\theta_1)$& $n_2$ & $\theta_1 \sim \mathcal{U}(0.5,2) $ & 1.00 (0.00) & 1.00 (0.00) \\ %
            \bottomrule
        \end{tabular}
    \label{tab:additive_model}
\end{table}

\begin{table}[h]
\caption{ \textbf{Details for the additive model with eleven components.} The noise terms are the same as in Table \ref{tab:additive_model}. Overall the model has 13 paramters. For the performance we report the mean and mean of the standard deviation on training with 1M datapoints across 5 optimization runs. }
      \centering
        \begin{tabular}{lclc}
            \toprule
            Model Component     &   Token    & Parameter Prior & \makecell[c]{Performance \\ $ q_\psi(M_i|x_o)$ }  \\  
            \midrule
            $\theta_1 \cdot t $ & $l_1$ & $\theta_1 \sim \mathcal{U}(-2,2) $ &0.68 (0.24)  \\ %
            $\theta_1 \cdot t $  &  $l_2$  & $\theta_1 \sim \mathcal{U}(-2,2) $   &0.69 (0.24)  \\ %
            $\theta_1 \cdot t^2 $  & $q_1$   & $\theta_1 \sim \mathcal{U}(-0.5,0.5) $  & 0.80 (0.24)  \\ %
            $ (\theta_1 + t)^2 $  &  $q_2$  & $\theta_1 \sim \mathcal{U}(-5,0) $   &0.98 (0.10)  \\ %
            $ \theta_1 \cdot t^3 $  &  $cub$  & $\theta_1 \sim \mathcal{U}(-0.1,0.1) $   &0.88 (0.21)  \\ %
            $\theta_1 \cdot \sin(\theta_2 t ) $ & $sin$ &\makecell[l]{$\theta_1 \sim \mathcal{U}(0,5)$ \\$\theta_2 \sim \mathcal{U}(0.5,5) $}  & 0.86 (0.23)  \\ %
            $\theta_1 \cdot \cos(\theta_2 t ) $ & $cos$ &\makecell[l]{$\theta_1 \sim \mathcal{U}(0,5)$ \\$\theta_2 \sim \mathcal{U}(0.5,5) $}  & 0.89 (0.21)  \\ %
            $ \theta_1$  &  $const_1$  & $\theta_1 \sim \mathcal{U}(-5,5) $   &0.72 (0.25)  \\ %
            $ \theta_1$  &  $const_2$  & $\theta_1 \sim \mathcal{U}(0,10) $   &0.80 (0.25)  \\ %
            noise$_1$: $n_{t_i} \sim \mathcal{N}(0,\theta_1)$ & $n_1$ & $\theta_1 \sim \mathcal{U}(0.1,2) $ & 1.00 (0.02)  \\ %
            noise$_2$: $n_{t_i} \sim (t_i+1) \mathcal{N}(0,\theta_1)$& $n_2$ & $\theta_1 \sim \mathcal{U}(0.5,2) $ & 1.00 (0.02)  \\ %
            \bottomrule
        \end{tabular}
    \label{tab:additive_modelBIG}
\end{table}

\paragraph{Prior}
To show the flexibility of the presented prior over model components, we defined a dynamically changing graph for the additive model. During a random walk, we increased the edge weights of the direct model paths to the end node with every sampled component and additionally decreased the weight between the linear components if one component is sampled by a factor of two. This favors simple models and disadvantages the co-occurrence of both linear components. This corresponds to rules $R^3_{xy}$, $R^3_{yx}$ and $R^4_{x\text{end}}$ in Appendix \ref{app:sec:model_prior}  with $c=0.5$ and $c=2$ respectively. The resulting empirical prior distribution is shown in grey in Figure \ref{fig:additive_model}b. 

The parameter priors for the model components are shown in Table \ref{tab:additive_model}.

\begin{table*}[h] 
\caption{\textbf{SBMI performance for the additive model for 500k training samples.} Comparison of SBMI and reference model posteriors in terms of Kullback-Leibler divergence (KL) and marginal performances. We calculated reference posteriors for 100 observations $x_o$ (see Table \ref{tab:additive_model} for performances of individual components). For the RMSE we used 1k observations $x_o$ and `Reference' corresponds to the RMSE between the observations $x_o$ and samples $x$ under the ground truth model and parameters. We report mean and standard deviation. }
\centering
\begin{tabular}{lccc}
\toprule
                 Measure &  Reference (Posterior) &    SBMI Posterior & Prior \\
\midrule
     KL &  -  & 0.28 (0.71) & 11.26 (1.88) \\ 
    Marginal Performance &  0.88 (0.15) &  0.86 (0.09) & 0.53 (0.12) \\
    RMSE  & 6.87 (6.05) & 7.05 (6.19) & 15.24  (7.95) \\
\bottomrule
\end{tabular}
\label{tab:additive_performance}
\end{table*}

\begin{table}[h]
\caption{ \textbf{Details for the additive model with 20 components.} The model prior is a fully connected graph for the additive components.  The noise components are mutually exclusive as in the smaller models. Overall the model has 28 parameters.}
      \centering
        \begin{tabular}{lcll}
            \toprule
            Model Component     &   Token    & Parameter Prior  \\  
            \midrule
            $\theta_1 \cdot t $ & $l_1$ & $\theta_1 \sim \mathcal{U}(-2,2) $ \\ 
            $\theta_1 \cdot t $  &  $l_2$  & $\theta_1 \sim \mathcal{U}(-2,2) $   \\ %
            $\theta_1 \cdot t^2 $  & $q_1$   & $\theta_1 \sim \mathcal{U}(-0.5,0.5) $  \\ %
            $ (\theta_1 + t)^2 $  &  $q_2$  & $\theta_1 \sim \mathcal{U}(-5,0) $   \\ %
            $ \theta_1 \cdot t^3 $  &  $cub$  & $\theta_1 \sim \mathcal{U}(-0.1,0.1) $   \\ %
            $\theta_1 \cdot \sin(\theta_2 t ) $ & $sin$ &\makecell[l]{$\theta_1 \sim \mathcal{U}(0,5)$ \\$\theta_2 \sim \mathcal{U}(0.5,5) $}  \\
            $\theta_1 \cdot \cos(\theta_2 t ) $ & $cos$ &\makecell[l]{$\theta_1 \sim \mathcal{U}(0,5)$ \\$\theta_2 \sim \mathcal{U}(0.5,5) $} \\
            $ \theta_1$  &  $const_1$  & $\theta_1 \sim \mathcal{U}(-5,5) $   \\
            $ \theta_1$  &  $const_2$  & $\theta_1 \sim \mathcal{U}(0,10) $  \\
            $ \theta_1\cdot \tanh(t-\theta_2)$  &  $tanh_1$  &\makecell[l]{$\theta_1 \sim \mathcal{U}(1,10)$ \\$\theta_2 \sim \mathcal{U}(2,8) $}  \\
            $ \theta_1\cdot \tanh(\theta_2-t)$  &  $tanh_2$  &\makecell[l]{$\theta_1 \sim \mathcal{U}(1,10)$ \\$\theta_2 \sim \mathcal{U}(2,8) $}  \\
            $ \theta_1\cdot \exp(-(t-\theta_2)^2)$  &  $g_1$  &\makecell[l]{$\theta_1 \sim \mathcal{U}(1,10)$ \\$\theta_2 \sim \mathcal{U}(2,8) $}  \\
            $ \theta_1\cdot \exp(-(t-\theta_2)^2/8$  &  $g_2$  &\makecell[l]{$\theta_1 \sim \mathcal{U}(1,10)$ \\$\theta_2 \sim \mathcal{U}(2,8) $}  \\
            $ \theta_1\cdot \text{ReLU}(t-\theta_2)$  &  $relu_1$  &\makecell[l]{$\theta_1 \sim \mathcal{U}(1,5)$ \\$\theta_2 \sim \mathcal{U}(2,8) $}  \\
            $ \theta_1\cdot \text{ReLU}(\theta_2-t)$  &  $relu_2$  &\makecell[l]{$\theta_1 \sim \mathcal{U}(1,5)$ \\$\theta_2 \sim \mathcal{U}(2,8) $} \\
            noise$_1$: $n_{t_i} \sim \mathcal{N}(0,\theta_1)$ & $n_1$ & $\theta_1 \sim \mathcal{U}(0.1,2) $ \\ %
            noise$_2$: $n_{t_i} \sim (t_i+1) \mathcal{N}(0,\theta_1)$& $n_2$ & $\theta_1 \sim \mathcal{U}(0.5,2) $ \\
            noise$_3$: $n_{t_i} \sim (11-t_i) \mathcal{N}(0,\theta_1)$& $n_3$ & $\theta_1 \sim \mathcal{U}(0.5,2) $ \\
            noise$_4$: $n_{t_i} \sim (t_i^2+1) \mathcal{N}(0,\theta_1)$& $n_4$ & $\theta_1 \sim \mathcal{U}(0.2,1) $ \\
            noise$_5$: $n_{t_i} \sim (11-t_i^2) \mathcal{N}(0,\theta_1)$& $n_5$ & $\theta_1 \sim \mathcal{U}(0.2,1) $ \\
        \bottomrule
        \end{tabular}
    \label{tab:bigger_model}
\end{table}

\paragraph{Network Details}
We used a one-dimensional convolutional network followed by fully connected layers as an embedding net for the additive model. 
The convolutional layers used a kernel size of five and stride one. The output of the last convolutional layer was flattened before passing it on to the fully connected network. All further parameters can be found in Table \ref{tab:additive_network_details} and \ref{tab:additive_network_detailsBIG}.

\begin{table}[h]\caption{\textbf{Network details for the additive model.} Square brackets indicate the layer-wise parameters, otherwise the same parameters were used for all layers. }
      \centering
        \begin{tabular}{lcccc}
            \toprule
            & Number of Layers & \makecell[c]{Dimensions / \\ \#Channels} & Components\\
            \midrule
            Convolutional layers & 2 & [10, 16] & - \\
            Fully connected layers &3 & [200, 200, 50] & -  \\
            MoGr net & 3& 80 & 3\\
            MDN net & 3 &  120 & 3 \\
            \bottomrule
        \end{tabular}
        \label{tab:additive_network_details}
\end{table}
\renewcommand\cellalign{ll}

\begin{table}[h]\caption{\textbf{Network details for the large additive model.} Square brackets indicate the layer-wise parameters, otherwise the same parameters were used for all layers. }
      \centering
        \begin{tabular}{lcccc}
            \toprule
            & Number of Layers & \makecell[c]{Dimensions / \\ \#Channels} & Components\\
            \midrule
            Convolutional layers & 2 & [10, 16] & - \\
            Fully connected layers &3 & [200, 200, 50] & -  \\
            MoGr net & 3& 120 & 3\\
            MDN net & 3 &  200 & 3 \\
            \bottomrule
        \end{tabular}
        \label{tab:additive_network_detailsBIG}
\end{table}
\renewcommand\cellalign{ll}

\subsection{Bayes Factors for the shown Example}\label{app:Bayes_factor_additive} For the example shown in Fig. \ref{fig:additive_model} we get Bayes factors of $B_{l_1l_2}=1.02 $ for the comparison of the two models with a single linear component (either $l_1$ or $l_2$), and $B_{l_1l_{12}}=1.45$,
if we compare the model with component $l_1$ with the one in which both model components are present ($l_{12}$). Following the scale by \cite{jeffreys1998theory} this would be `inconclusive' about the preference of the models.

\FloatBarrier
\subsection{DDM} \label{app:model_details_ddm}

Drift diffusion models can be described as a stochastic differential equation for a decision variable $z$:
\begin{equation*}
 \text{d}z = d(z,t)\text{d}t + \text{d}W ,
\end{equation*}
with initial condition $z_0$, drift term $d$, and a Wiener noise process $W$. A decision is taken when the decision variable hits the boundary $|d(z,t)|\geq b(t)$ (Figure  \ref{fig:ddm}a). An additional parameter delays the starting time of the process (`non-decision time').  

We included two different drift terms $d$:
\begin{enumerate}
    \item constant drift: $d(z,t) = \theta_1$, \quad and
    \item  leaky drift: $d(z,t) = \theta_1 + \theta_2 \cdot z  $ \quad (with $\theta_2<0$), 
\end{enumerate}
and two boundary conditions $b$:
\begin{enumerate}
    \item constant boundary: $b(t) = \theta_1$, \quad and
    \item exponentially collapsing boundary: $b(t) = \theta_1 - \text{exp}(-t /\theta_2) $. 
\end{enumerate}

The initial condition $z_0$ was fixed to be zero and the noise term had a constant standard deviation of one. 
The non-decision time was a free parameter but was present in all models (see Figure \ref{fig:ddm}b).

\begin{table}[h]\caption{\textbf{Details for the DDM.} We used independent uniform distributions $\mathcal{U}$ for all parameter priors. The performances were calculated on 1k samples from the prior distribution, and we report mean and standard deviation. }
      \label{tab:ddm_details}
      \centering
        \begin{tabular}{lclcc}
            \toprule
            Model Component & Token    & Parameter Prior   & \makecell[c]{Performance \\  Model Posterior}& \makecell[c]{Performance \\  Model MAP}       \\
            \midrule
            constant drift & $d_c$ & $\theta_1 \sim \mathcal{U}(0,5) $   &0.85 (0.23)&  0.90 (0.31)  \\
            leaky drift  & $d_l$   & \makecell[l]{$\theta_1 \sim \mathcal{U}(0,5)$ \\$\theta_2 \sim \mathcal{U}(-20,-5) $}  &0.85 (0.23) & 0.90 (0.31)\\
            constant boundary  & $b_c$   & $\theta_1 \sim \mathcal{U}(0.3,2) $    &0.90 (0.20) &  0.92 (0.27) \\
            exp. collapsing boundary & $b_{exp}$ & \makecell[l]{$\theta_1 \sim \mathcal{U}(0.3,2)$ \\$\theta_2 \sim \mathcal{U}(0.5,1.5) $} &0.90 (0.20) & 0.92 (0.27)\\
            non-decision time & $ndt$ & $\theta_1 \sim \mathcal{U}(0.1,0.3)$ & 1.00 (0.00) & 1.00 (0.00)\\
            \bottomrule
        \end{tabular}
\end{table}

\paragraph{Training Data}
We used the \emph{pyDDM} toolbox \citep{shinn2020pyddm} to solve the DDM numerically for every $\theta$ using the Fokker-Planck equation. From the approximated decision time and choice distribution we then sampled 400 iid trials for each $\theta$. This results in a $400 \times 2$ data matrix with the recorded continuous decision times and binary decisions.

As training data, we sampled 200k models from the prior, solved these DDMs and drew 400 trials. 
From this data, we excluded datapoints with more than 300 undecided trials (defined as trials with a decision time larger than 10 seconds). From the remaining $\approx$180k datapoints we hold back 1k test datapoints and divided the other part into 10\% validation and 90\% training data.

\paragraph{Prior}
Initial experiments showed that models with leaky drift and constant boundary conditions often resulted in unrealistically long decision times (>10sec), and we therefore discouraged their co-occurrence by including a negative coupling between these two terms in the model prior.

All edges of the model prior have initially the same weight in the shown prior graph (Figure \ref{fig:ddm}b). If the leaky drift component is visited in a random walk, the edge weight of the constant boundary condition is decreased by a factor of two (corresponding to $R^2_{xy}$).  

The parameter priors for the different model components are shown in Table \ref{tab:ddm_details}

\paragraph{Network Details}
To account for the iid trial structure of the DDM data, we used a permutation invariant embedding net \cite{chan2018likelihood,radev2020bayesflow}.
 In this setup, the single-trial data is first processed by a fully connected network, mean pooled, and then passed through additional fully connected layers.

Each trial (represented as a vector (decision time, decision)) is first processed by the `single trial net', which we implemented as a fully connected neural network. The output is then averaged (making it permutation invariant) and passed on to a second fully connected network. 
The used hyperparameters (Table \ref{tab:ddm_network_details}) were the best hyperparameters in a coarse hyperparameter sweep over eight models, varying three hyperparameters.
To this end, we used \emph{Optuna} \cite{optuna_2019} and varied the embedding dimensions (last layer of the single trial net and the last layer of the fully connected embedding net) and the dimension of the MoGr net. 

\begin{table}[h]\caption{\textbf{Network details for the DDM.} Square brackets indicate the layer-wise parameters, otherwise the same parameters were used for all layers. }
      \centering
        \begin{tabular}{lcccc}
            \toprule
            & Number of Layers & Dimensions & Components \\
            \midrule
             Single trial net & 3 & [120, 120, 100]& - \\
             Fully connected embedding net  &3& [120, 120, 30] & - &\\
            MoGr net & 3& 80 & 3 \\
            MDN net &  3 & 120   & 3  \\
            \bottomrule
        \end{tabular}
        \label{tab:ddm_network_details}
\end{table}

\paragraph{Dataset}
The used data \cite{roitman2002data} was collected from two monkeys performing a random dot motion discrimination task. Visual stimuli of moving dots with different coherence rates (0, 3.2, 6.4 and 12.8\%) were presented and the monkeys had to decide on the moving direction. 
We randomly subsampled 400 trials for each stimulus condition to match the dimension of our training data and show the results for `monkey N' throughout the manuscript. The dataset can be found here: \url{https://shadlenlab.columbia.edu/resources/RoitmanDataCode.html}.

\begin{table}[ht]\caption{\textbf{DDM parameter comparison for example observation.} 
Sample mean and standard deviation for 10k samples from the SBMI parameter posterior for the example observation from Figure \ref{fig:ddm}. The model posteriors are  $q_{\psi}(\text{gt-model}|x_o)=0.75$ and $q_{\psi}(\text{c.-drift-model}|x_o)=0.25$. }
      \label{tab:ddm_example_obs}
      \centering
        \begin{tabular}{lcccc}
            \toprule
            Model Component & Parameter & \makecell[c]{Ground\\ Truth}   &  \makecell[c]{SBMI posterior \\  | gt-model}      &\makecell[c]{SBMI posterior \\  | c.-drift-model}\\
            \midrule
            constant drift  & $\theta_1$ & - & - &  1.37 (0.08)   \\
            leaky drift     & \makecell[c]{$\theta_1 $ \\$\theta_2 $} &\makecell[c]{ 2.00 \\ -10.00} &\makecell[c]{ 1.79 (0.17) \\ -9.71 (3.60)} & \makecell[c]{ - \\ - }\\
            constant boundary    & $\theta_1$  & -  & -  &  - \\
            exp. collapsing boundary  & \makecell[c]{$\theta_1 $ \\$\theta_2$} &\makecell[c]{ 0.70 \\ 0.70} & \makecell[c]{ 0.75 (0.13) \\ 0.76 (0.11)} & \makecell[c]{ 1.73 (0.15) \\ 1.07 (0.11)} \\
            non-decision time & $\theta_1 $ & 0.25 & 0.22 (0.03) & 0.14 (0.02)\\
            \bottomrule
        \end{tabular}
\end{table}

\begin{table} [h]
\caption{\textbf{DDM predictive performance for experimental data.} Comparison of mean decision times $\mu$ and standard deviation of decision times $\sigma$ of the experimental data ($\hat{\cdot}$) to posterior predictive samples. We report the mean and standard deviation for the different measures based on 10k SBMI posterior samples and for ten \emph{pyDDM} fits with different random seeds. The statistics are pooled over the different coherence rates. }
\centering
\begin{tabular}{lll}
\toprule
                 Measure &   \emph{pyDDM} &    SBMI \\
\midrule
    decision time: $| \mu - \hat \mu|$ & 0.06 (0.06) & 0.06 (0.06)  \\  %
    decision time: $| \sigma - \hat \sigma|$ & 0.17 (0.15) & 0.13 (0.15) \\ %
    deviation correct trials in \%  & 2.08 (1.75) &  2.22 (1.83) \\
    MSE on densities ($\cdot 10^{-2}$) & 9.66 (9.18) & 9.66 (8.94) \\ 
\bottomrule
\end{tabular}
\label{tab:ddm_experimental}
\end{table}

\begin{table} [h]
\caption{\textbf{DDM predictive performance for experimental data for individual coherence rates.} We report the mean and standard deviation for the MSE based on 10k SBMI posterior samples and for ten \emph{pyDDM} fits with different random seeds. See Table S1 (main paper) for the pooled statistics.}
\centering
\begin{tabular}{llll}
\toprule
                 Measure &  Coherence \% & \emph{pyDDM} &    SBMI \\
\midrule
    MSE on densities ($\cdot 10^{-2}$)&0.0 & 23.80 (0.004) & 23.79 (0.010) \\ 
 & 3.2  & 10.66 (0.002) & 10.67 (0.009)  \\
     & 6.4  & 3.64 (0.001) & 3.64 (0.008)  \\
     & 12.8  & 0.55 (0.001) & 0.56 (0.011)  \\
\bottomrule
\end{tabular}
\label{tab:ddm_experimental_detailed}
\end{table}

\FloatBarrier

\subsection{Hodgkin-Huxley model} \label{app:HH_details}

We implemented a version of the Hodgkin-Huxley model based on \cite{pospischil2008minimal} in \emph{JAX} \cite{jax2018github}.

The differential equations are given by
\begin{align*}
\frac{dV}{dt} = &g_L(E_L-V)+
                    {g}_{Na} m^3h(E_{Na}-V) 
  +{{g}_{K}}n^4(E_K-V)+ {g}_M p(E_K-V) + g_{Ca} q^2 r ( E_{Ca} - V ) \\
    &+ I_{inj} +\sigma\eta(t),
\end{align*}
and 
\begin{equation*}
   \frac{ds}{dt} =\frac{s_\infty\left(V\right)-s}{\tau_s\left(V\right)};  s\in\{m,h,n,p,q,r\}, 
\end{equation*}
where $I_{inj}$ corresponds to the injected current to stimulate the cell.
The parameters $V_t$ for $K$ and $\tau$ for $K_m$ define the steady state $s_\infty$ and the time constant $\tau_s$ for the corresponding gating parameters. All details can be found in \cite{pospischil2008minimal}.

For the noise term $\sigma \eta(t)$ we sampled for each time step independent noise and scaled it corresponding to the standard deviation $\theta_1 = \sigma$. We used a step current $I_{inj}$ of $2 \mu A / cm^2$ for $1000ms$ and run the simulation for $1450 ms$.  This stimulus and recording protocol corresponds to the voltage recordings from the Allen Cell database.  

\paragraph{Prior}
The graph for the model prior is shown in Fig. \ref{fig:app:HHprior}, and all edge weights were set to one, except the edge from $K$ to $Na$ which was set to $1/3$.
The parameter priors for the different model components are shown in Table \ref{tab:hh_model} and adapted from \cite{gonccalves2020training}.

\paragraph{Training data}
We sampled and simulated 100k models. We excluded simulations with \emph{Nan} and \emph{inf} values, resulting in a training dataset of 99,895 simulations. 
We used these simulations to calculate 24 summary statistics and imputed \emph{Nan} values either by the appropriate \emph{max}, \emph{min} or \emph{mean} value.  
The used summary statistics were adapted from previously used summary statistics in \cite{scala2021phenotypic}.

\paragraph{Network details}
The network details can be found in Table \ref{tab:HH_network_details}.

\begin{table}[h]
\caption{ \textbf{Details for the Hodgkin-Huxley model.} The parameter $\theta_1$ in the noise terms defines the standard deviation of a normal distribution $\mathcal{N}$, and $\mathcal{U}(a,b)$ defines a uniform distribution on the interval $[a,b]$. For the performance we report the mean and standard deviation over all test samples $x_o$ as well as over test samples which have at least one spike $x_o^s$.}
          \label{tab:hh_model}
      \centering
        \begin{tabular}{lclccc}
            \toprule
            Model Component     &   Token    & Parameter Prior  & \makecell[c]{Performance (all) \\ $q_\psi(M_i|x_o)$}& \makecell[c]{Performance (spikes) \\  $q_\psi(M_i|x_o^s)$}    \\
            \midrule
            Leak current& $I_L$ & \makecell{$g_L \sim \mathcal{U}(10^{-6},3\cdot 10^{-4}) $ \\ $E_L\sim \mathcal{U}(-80, -60) $} &1.00 (0.00) & 1.00  (0.00) \\ %
            Potassium channel  &  $K$  & \makecell{$g_K \sim \mathcal{U}(1.5\cdot 10^{-3},1.5\cdot 10^{-2}) $ \\ $V_t \sim \mathcal{U}(-70,-50) $ }  &1.00 (0.00) & 1.00 (0.00) \\ %
            M-type potassium channel  & $K_m$   & \makecell{$g_M \sim \mathcal{U}(10^{-5},6\cdot 10^{-4}) $ \\ $\tau \sim \mathcal{U}(200,2000)  $ }  & 0.96 (0.12) & 0.98 (0.01) \\ %
            Sodium channel& $Na$ & $g_{Na} \sim \mathcal{U}(8\cdot 10^{-3},8\cdot 10^{-2}) $ &1.00 (0.00) & 1.00  (0.00) \\ %
            Calcium channel & $Ca$ & $g_{Ca} \sim \mathcal{U}(5\cdot 10^{-5}, 10^{-3}) $ & 0.74 (0.25) & 0.94 (0.18) \\ %
            Noise & $Noise$ & $\theta_1 \sim \mathcal{U}(10^{-4},1.5\cdot 10^{-1}) $ & 1.00 (0.00) & 1.00 (0.00) \\ %

            \bottomrule
        \end{tabular}
    \label{tab:HH_model}
\end{table}

\begin{table}[h]\caption{\textbf{Network details for the Hodgkin-Huxley model.} Note that we replaced the embedding net by summary statistics in this case. }
      \centering
        \begin{tabular}{lcccc}
            \toprule
            & Number of Layers & Dimensions & Components \\
            \midrule
            MoGr net & 3& 80 & 3 \\
            MDN net &  3 & 150   & 3  \\
            \bottomrule
        \end{tabular}
        \label{tab:HH_network_details}
\end{table}

\clearpage
\section{Supplementary Figures}
\vspace{4cm}
\begin{figure}[h]
    \includegraphics[width=1\textwidth]{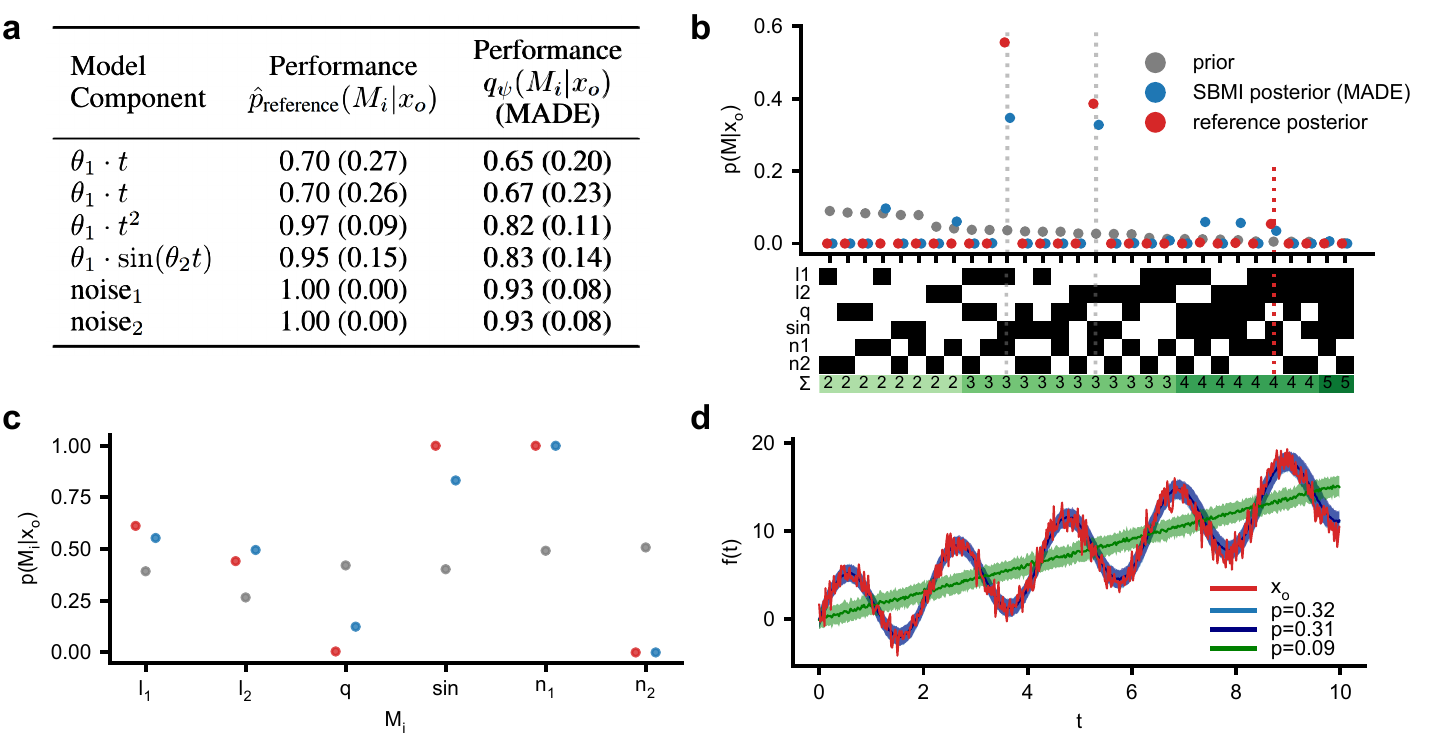}
    \caption{\textbf{SBMI on the Additive Model} (using MADE instead of MoGr). \textbf{(a)} Table with marginal posterior performance for the different model components of the additive model. Initial experiments showed a worse performance of the posterior implemented as MADE compared to a MoGr  (Table \ref{tab:additive_model}).
    \textbf{(b)} Model posterior implemented as MADE conditioned on $x_o$ shown in (d) (similar to Fig. \ref{fig:additive_model}b).
    \textbf{(c)} Marginal model posterior implemented as MADE for the same observation.
    \textbf{(d)} Posterior predictives (and local uncertainties as mean $\pm$ std.) of the three most likely models, covering 72\% of the model posterior mass. Compared to Fig. \ref{fig:additive_model}d models without the sinusoidal get non-negligible posterior mass. 
    }
        \label{app:fig:made_results}
\end{figure}

\begin{figure}[h] 
  \centering
  \includegraphics{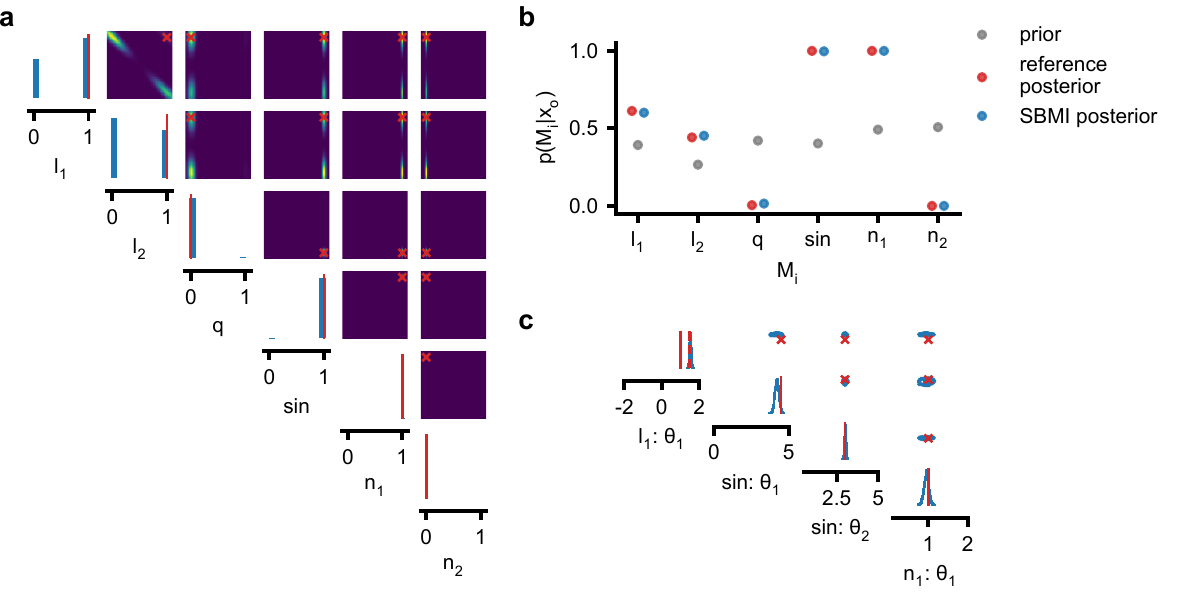}
  \caption{
    \textbf{SBMI posterior for the additive model. (a)}  Smoothed one- and two-dimensional marginal distribution for the (binary) model posterior. The ground truth model is indicated in red. 
    \textbf{(b)} Marginal model posterior distribution for the observation $x_o$ shown in Figure \ref{fig:additive_model}. The ground truth model consists of the components $l_1,~l_2, ~\sin$, and $n_1$.
    \textbf{(c)} One- and two-dimensional marginal distribution of the SBMI parameter posterior given the MAP model. The ground truth parameters are indicated in red. The sum of the coefficients $\theta_1$ for the two linear components $l_1$ and $l_2$ of the ground truth model is indicated as a dashed line in the most left plot. It matches the mean of the posterior marginal for $l_1$.
    }
    \label{fig:app:additive}
\end{figure}

\begin{figure}[h]
  \centering
  \includegraphics{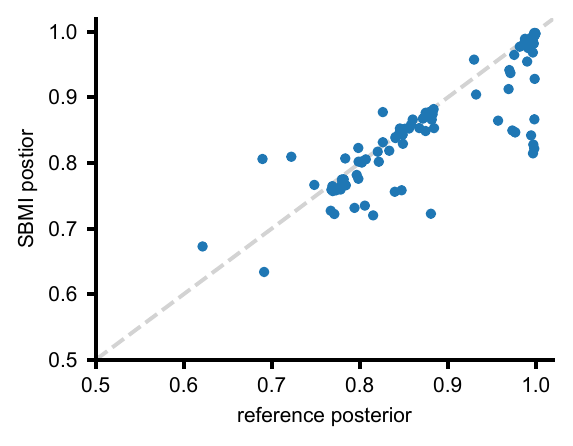}
  \caption{
    \textbf{Model posterior for the additive model.} Mean marginal performance of the reference posterior vs the SBMI posterior for 100 examples of the additive model with a correlation of 0.85. Note that uncertainty in the model posterior is reflected by lower `performance' values which is not necessarily a bad sign as it might reflect the true underlying uncertainty. 
    The high correlation indicates that the uncertainty of both posteriors is similar for most examples, which validates our SBMI approach. 
    }
    \label{fig:app:additive_correlation_mmp}
\end{figure}

\begin{figure}[h]
  \centering
  \includegraphics[]{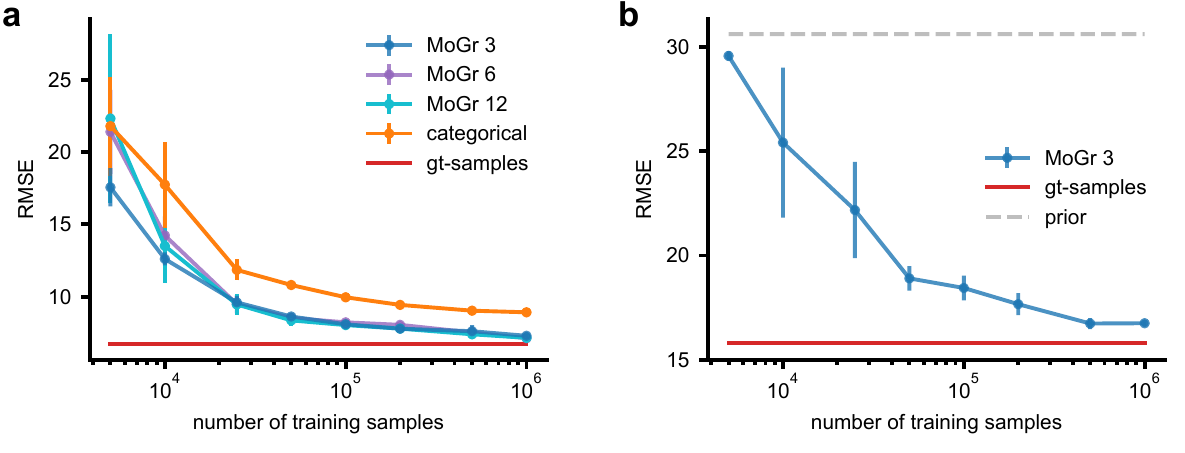}
  \caption{
    \textbf{SBMI performance is robust to the number of mixture components and scales to larger models.} 
     \textbf{(a)} Posterior predictive performance for the large additive model with 11 components, and $2^{11}=2048$ possible combinations of model components in terms of RMSE. We compare different choices for the distribution-family of the model posterior (either a mixture of Grassmann distributions with varying number of components (3, 6 and 12) (MoGr) or a categorical distribution). At the same time we increase also the number of components for the mixture of Gaussian parametere posterior network. 
    \textbf{(b)} Posterior predictive performance for an additive model with 20 components (specified in Table \ref{tab:bigger_model}) and $2^{20}\approx 1M$ possible combinations of model components. A categorical distribution can not be fitted anymore as we have very few to no samples for each combination of model components in a dataset of size $1M$.
    }
    \label{fig:app:additive_BIG_comparison_cat}
\end{figure}

\begin{figure}[h]
  \centering
  \includegraphics{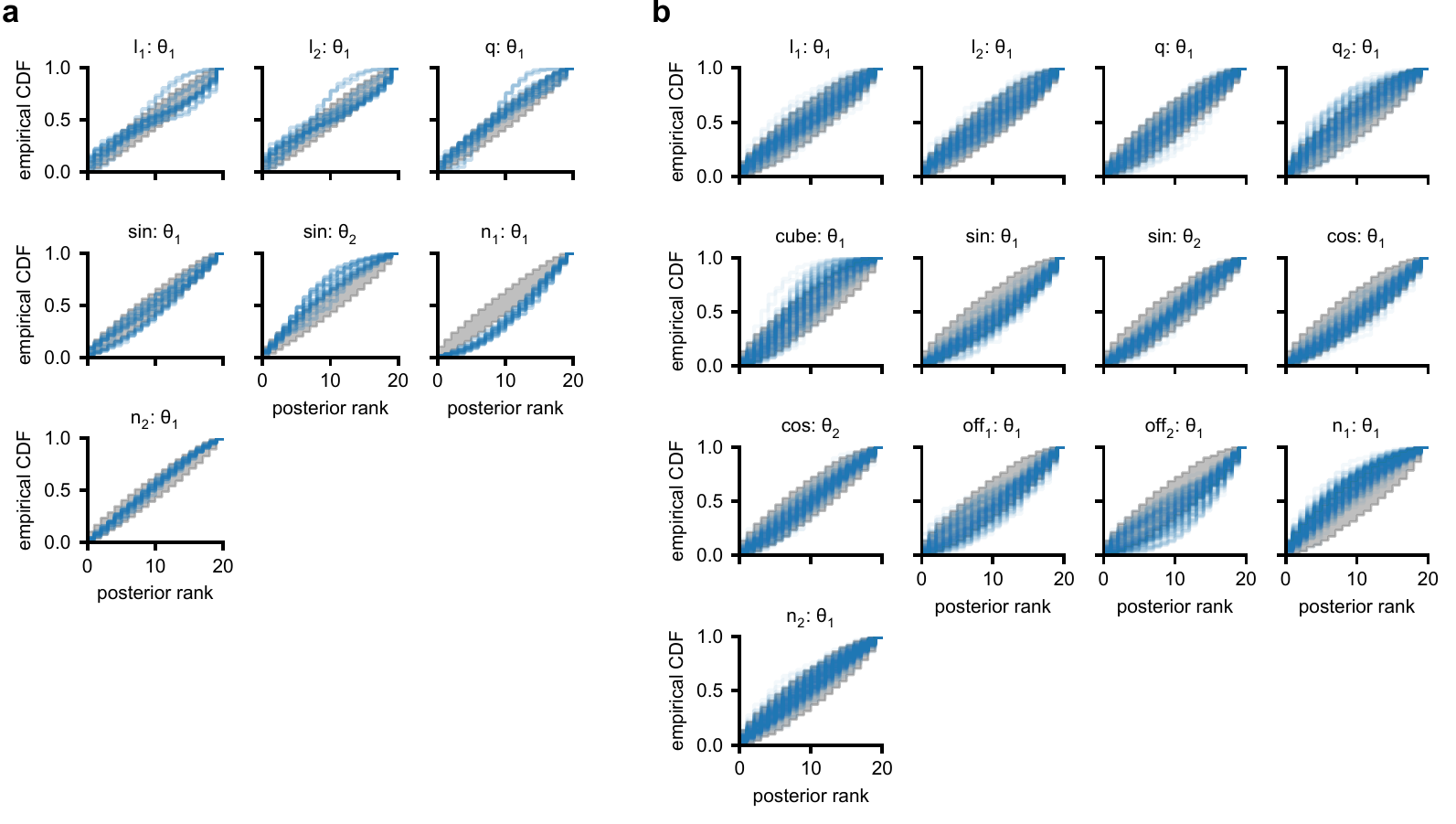}
  \caption{
    \textbf{Simulation-based calibration of the parameter posterior for the additive models.} 
    \textbf{(a)}  Posterior calibration of the \emph{small} additive model, by individual model parameters for all possible model component combinations. Grey regions indicate the 99\% confidence intervals of a uniform distribution, given the provided number of samples. 
    \textbf{(b)} Same as (a) for the \emph{large} additive model and for all models with at least 50 samples in the test dataset of 100k samples. 
    For all plots we ranked the true parameter $\theta_o$ against 1k posterior samples. 
    }
    \label{fig:app:SBC_additive}
\end{figure}

\begin{figure}[h]
  \centering
  \includegraphics{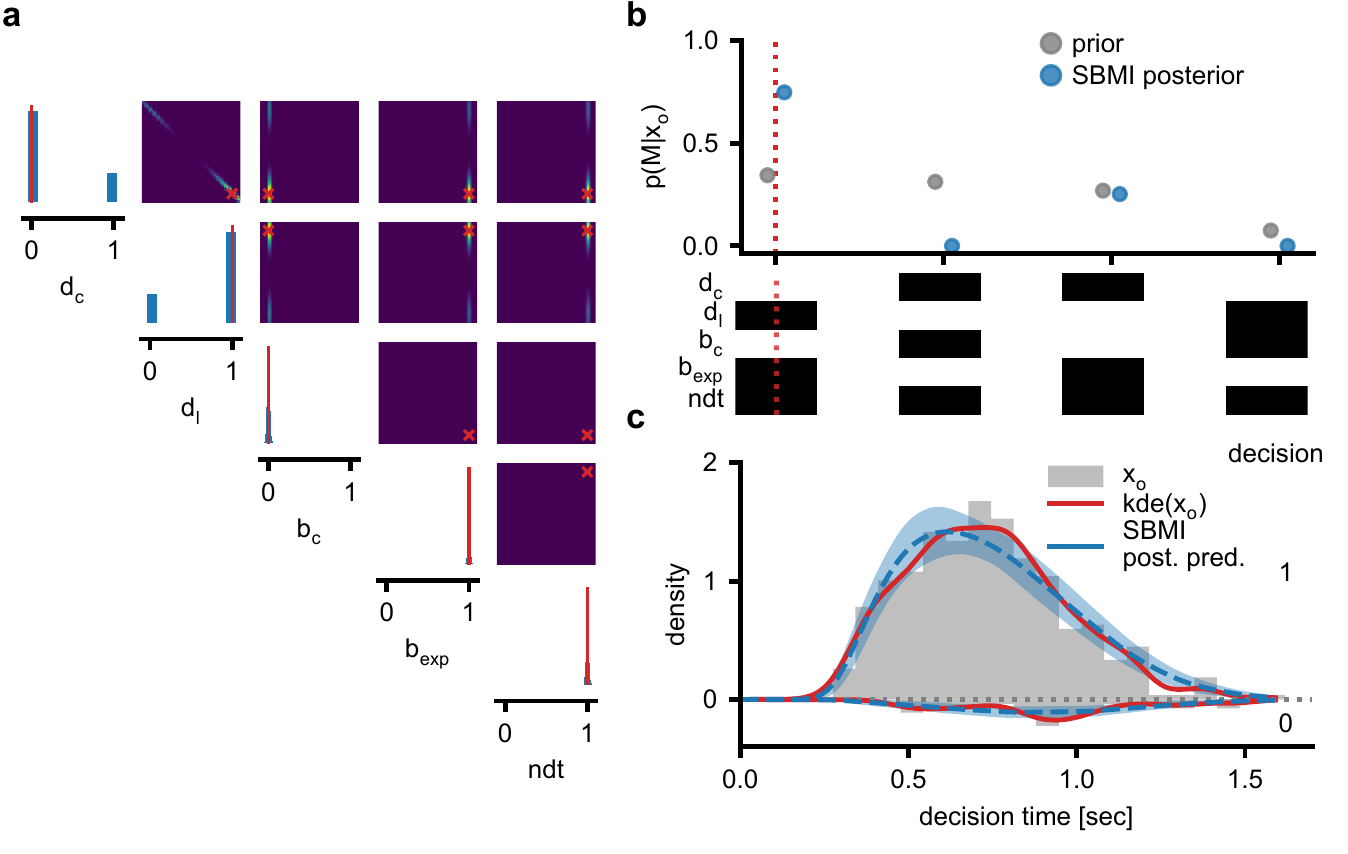}
  \caption{
   \textbf{DDM posterior and posterior predictives. (a)}  Smoothed one- and two-dimensional marginal distribution for the (binary) model posterior. The ground truth model is indicated in red. 
    \textbf{(b)} Model prior and SBMI model posterior with the ground truth model indicated as a red dotted line. 
    \textbf{(c)} Posterior predictives for the example observation of Figure \ref{fig:ddm}. The global uncertainty is shown as mean $\pm$ std. over predictions from different model posterior samples. 
    }
    \label{fig:app:ddm}
\end{figure}

\begin{figure}[h]
  \centering
  \includegraphics{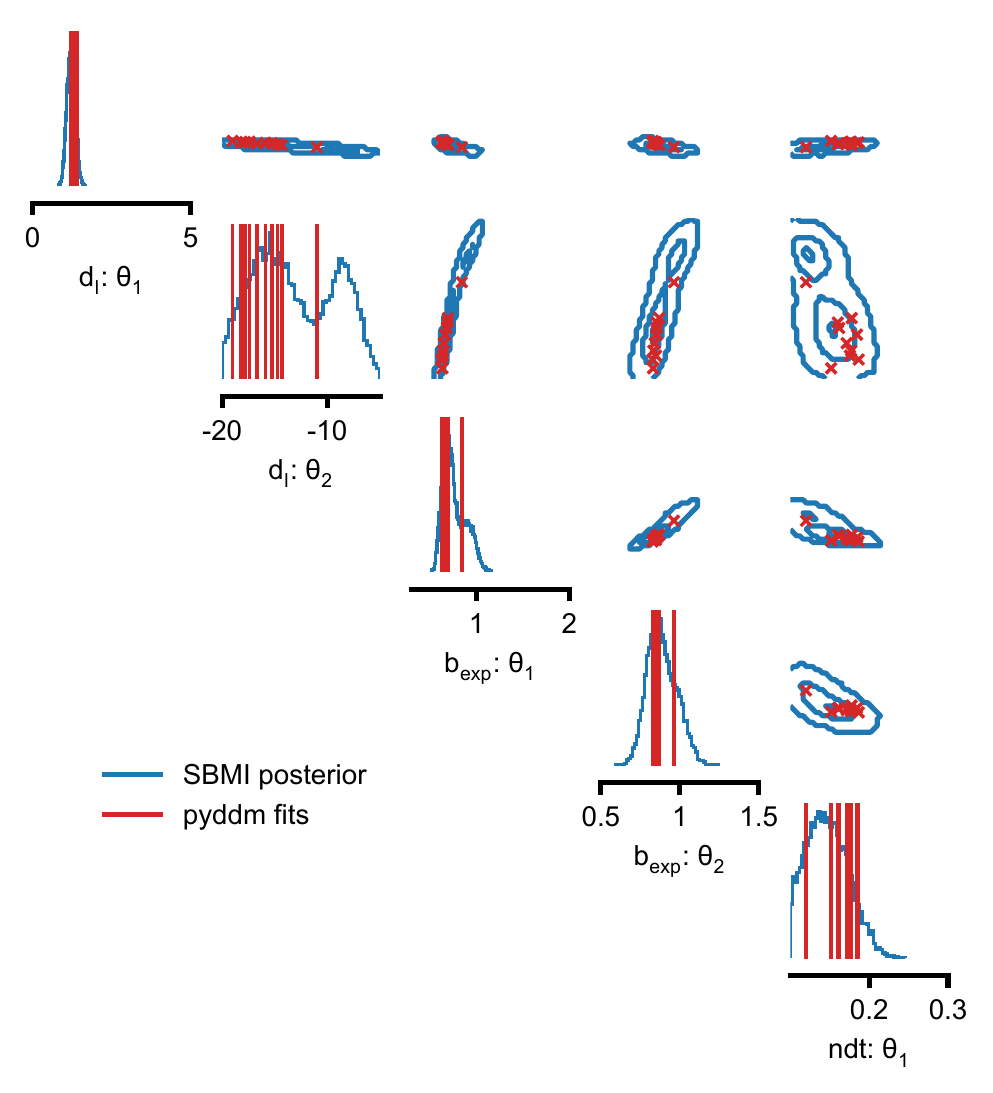}
  \caption{
    \textbf{Parameter posterior distribution for the DDM on experimental data.}   One- and two-dimensional marginals of the SBMI parameter posterior for a coherence rate of 6.4\%.  Red markers indicate ten \emph{pyDDM} fits for a fixed model with different random seeds, all resulting in similar loss values. 
    }
    \label{fig:app:ddm_2d_pyddm}
\end{figure}

\begin{figure}[h]
  \centering
  \includegraphics{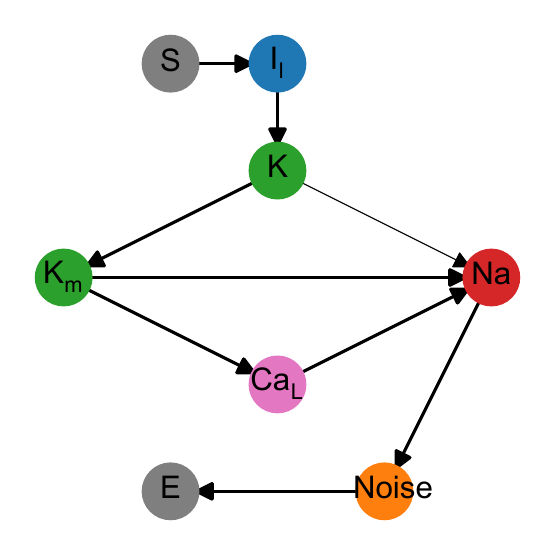}
  \caption{
    \textbf{Model prior for the Hodgkin-Huxley model.} Model prior graph. While potassium and sodium channels are always present in this definition, the presence of m-type potassium and calcium channels is inferred from the observation. This reflects our prior knowledge for spiking neurons in which sodium and potassium channels are essential for the spiking mechanism. The sampling distribution is shown in Fig. \ref{fig:app:HH_posterior}. 
    }
    \label{fig:app:HHprior}
\end{figure}

\begin{sidewaysfigure}
     \centering
    \includegraphics[width=0.95\textheight]{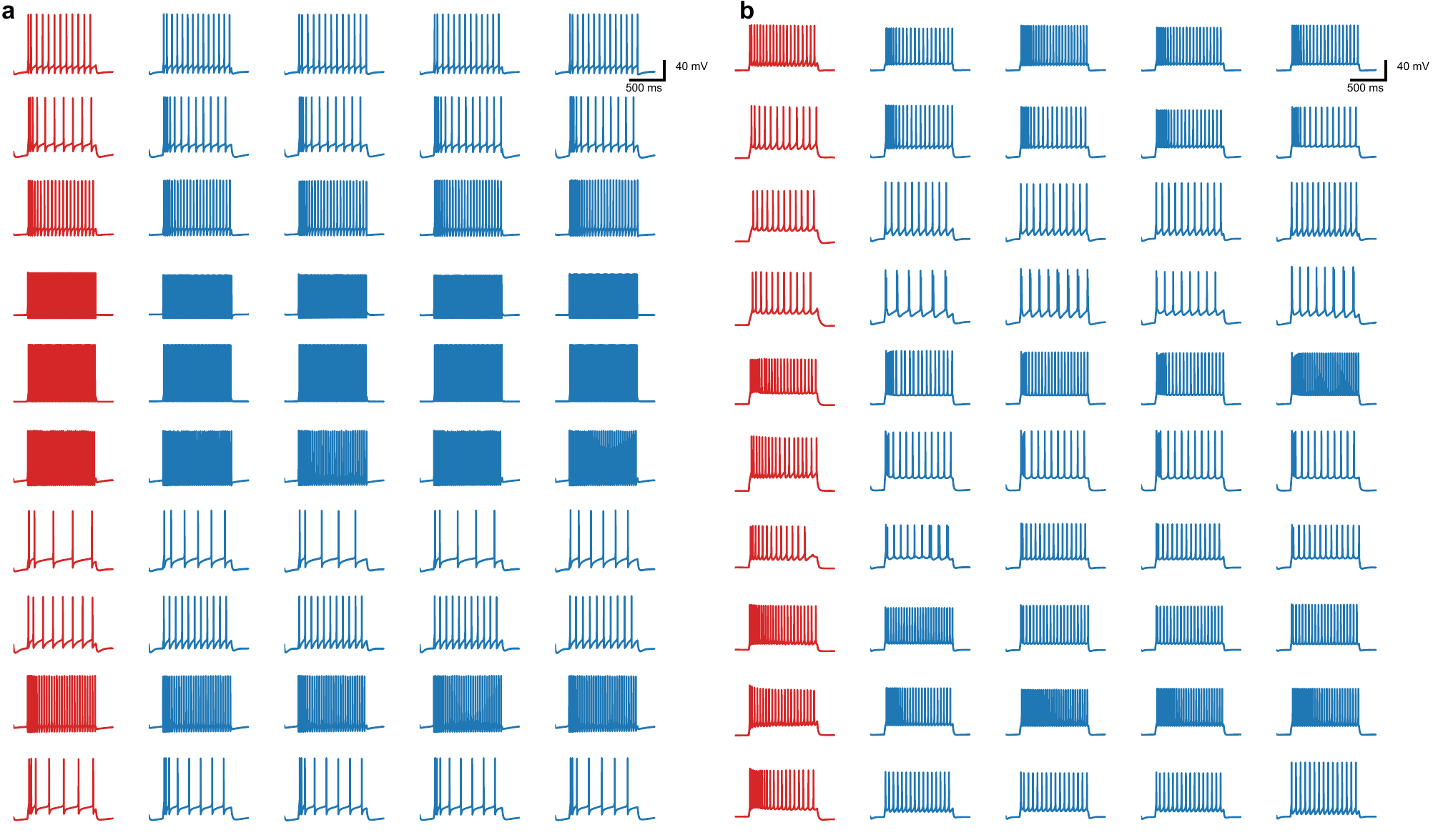}
        \caption{ 
            \textbf{Posterior predictives of the Hodgkin-Huxley model.} \textbf{(a)} Ten synthetic samples (red) with four posterior predictive samples each (blue).\\
            \textbf{(b)} Ten voltage recordings from the Allen Celltype database with four posterior predictive samples each. }
    \label{fig:app:HH_traces}
\end{sidewaysfigure}

\begin{figure}[h]
  \centering
  \includegraphics[]{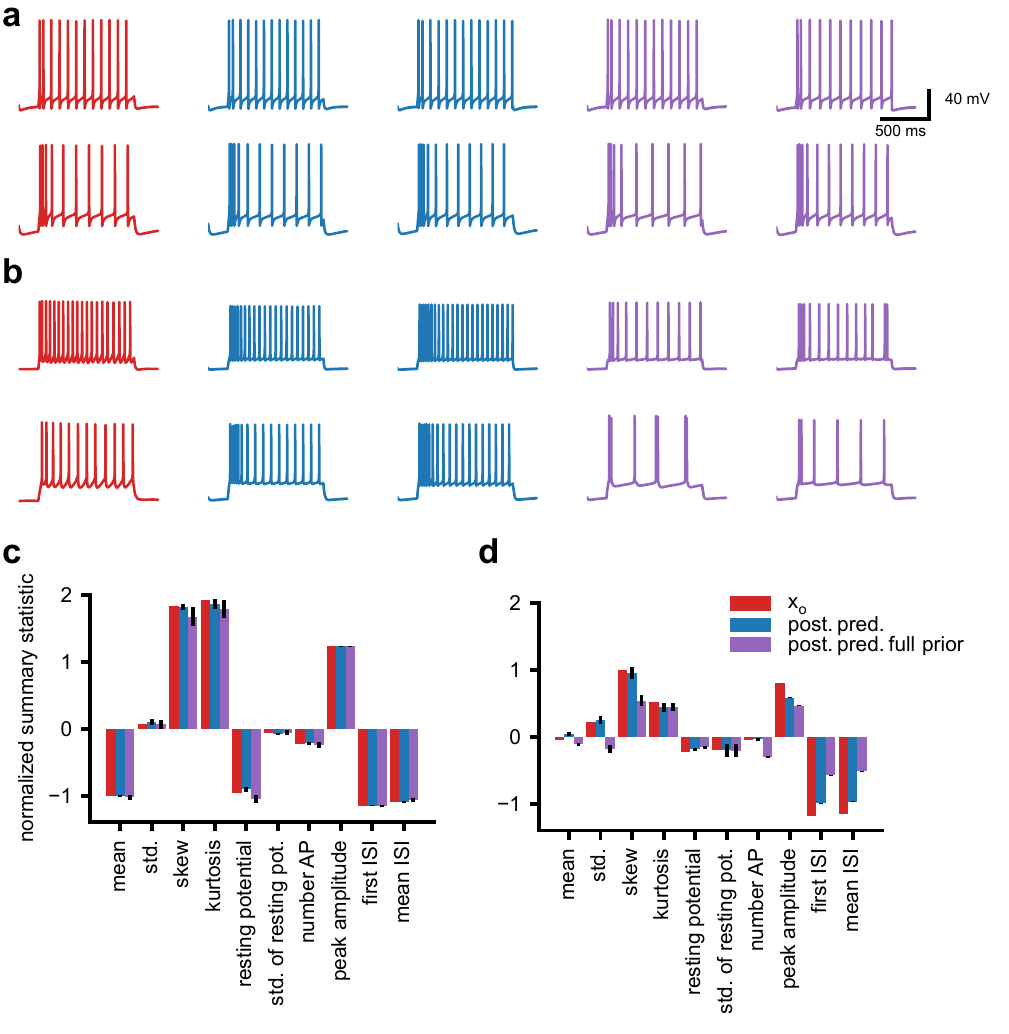}
  \caption{
    \textbf{Leveraging domain knowledge enhances posterior performance for the Hodgkin-Huxley model.} 
    \textbf{(a)} Two synthetic samples (red) with two posterior predictive samples each (blue) (same as Fig. \ref{fig:HH_main} a) with additional two posterior predictives from a model trained on a fully connected model prior (violet). 
    \textbf{(b)} Two voltage recordings from the Allen Cell database (red) with two posterior predictive samples each for the standard model (blue) (same as Fig. \ref{fig:HH_main} b)  with additional two posterior predictives from a model trained on a fully connected model prior (violet). 
     \textbf{(c)} Ten example summary statistics (out of 24) for the upper trace in (a) and summary statistics for ten posterior predictive samples from the respective model (mean$\pm$std.).
      \textbf{(d)} Ten example summary statistics for the voltage recording from the Allen dataset shown in (b), upper trace, and summary statistics for ten posterior predictive samples from the respective model (mean$\pm$std.). 
    }
    \label{fig:app:HH_susmstats}
\end{figure}

\begin{figure}[h]
  \centering
  \includegraphics{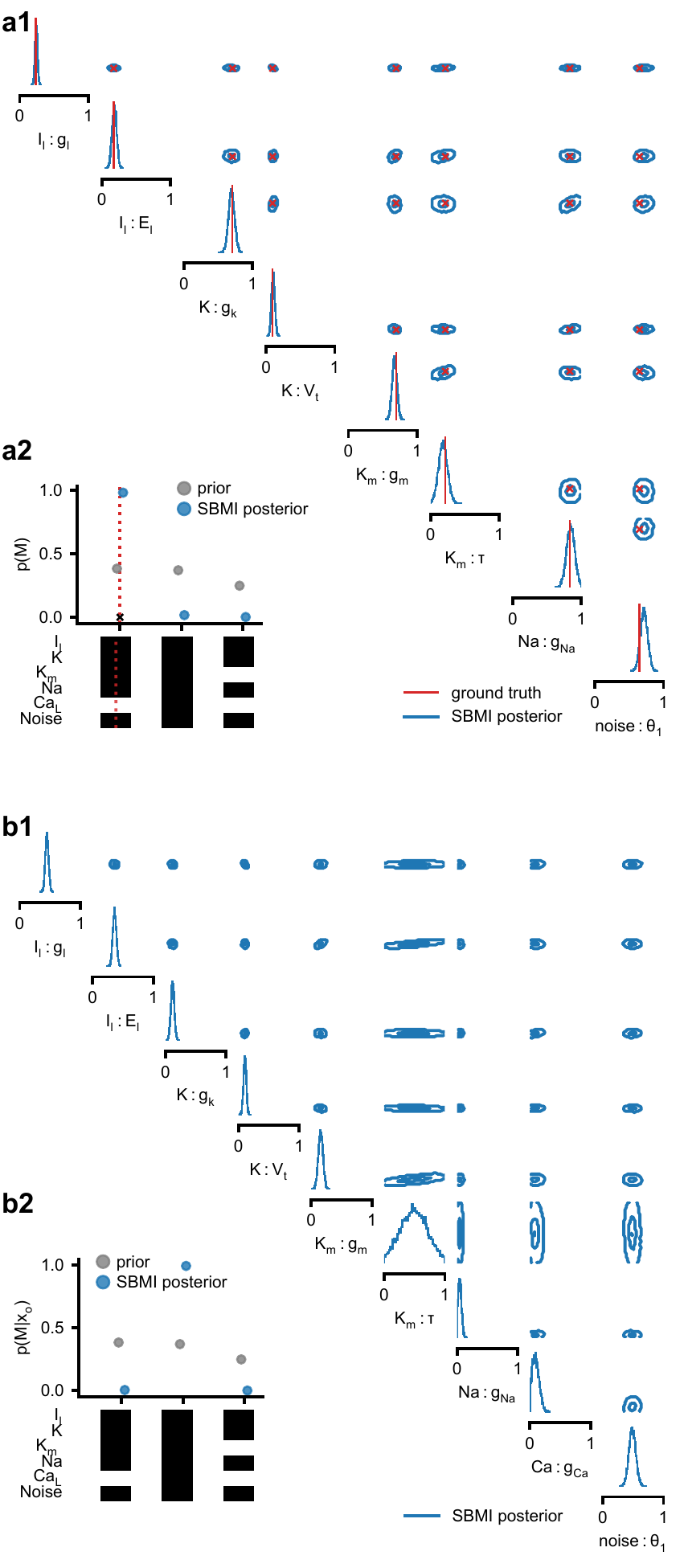}
  \caption{
\textbf{Model and parameter posteriors for the Hodgkin-Huxley model. (a1)} 
    Normalised one and two dimensional marginal distributions for the parameter posterior for the synthetic example shown in Fig. \ref{fig:HH_main} a (upper trace). Red dots/lines are indicating the ground truth parameter $\theta_o$. 
    \textbf{(a2)} The corresponding model posterior for the example in (a1). Red dotted line indicates the ground truth model.
    \textbf{(b1)} Normalised one and two dimensional marginal distributions for the parameter posterior for the voltage recording from the Allen database shown in Fig. \ref{fig:HH_main} b (upper trace).
    \textbf{(b2)} The corresponding model posterior for the example in (b1). 
    }
    \label{fig:app:HH_posterior}
\end{figure}

\begin{figure}[h]
  \centering
  \includegraphics[]{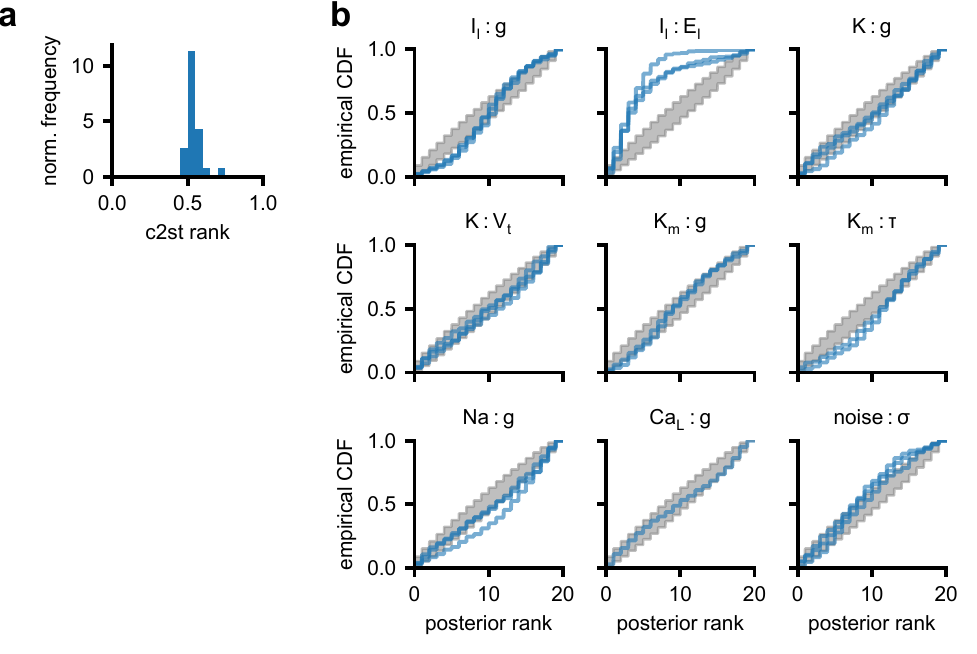}
  \caption{
    \textbf{Simulation-based calibration of the parameter posterior for the Hodgkin-Huxley model.} 
    \textbf{(a)} Histogram of the c2st ranks. A value of 0.5 indicates a well calibrated posterior for which the rank statistics are  indistinguishable from a uniform distribution. 
    \textbf{(b)} Posterior calibration by individual model parameters for all possible model component combinations. Grey regions indicate the 99\% confidence intervals of a uniform distribution, given the provided number of samples. 
    For all plots we ranked the true parameter $\theta_o$ against 1k posterior samples. 
    }
    \label{fig:app:HH_SBC}
\end{figure}

\end{document}